\documentclass{article}

\usepackage[preprint]{neurips_2020}
\PassOptionsToPackage{numbers, compress}{natbib}

\usepackage{amsxtra}     
\usepackage{amsthm}
\usepackage{amssymb}
\usepackage{amsfonts}
\usepackage{graphicx}    
\usepackage{rotating}
\usepackage{color}
\usepackage{epsfig}
\usepackage{verbatim}
\usepackage{setspace}    

\usepackage{CJKutf8}
\usepackage{comment}
\usepackage{mathtools}
\usepackage{enumitem}
\usepackage{subfig}
\usepackage{epstopdf}
\usepackage{easybmat}
\usepackage{arydshln}
\usepackage{xcolor}
\usepackage[colorlinks=true, linkcolor=blue, citecolor=blue, breaklinks=true]{hyperref}
\usepackage{hanging}

\usepackage{scalerel,stackengine}
\stackMath
\newcommand\widecheck[1]{%
\savestack{\tmpbox}{\stretchto{%
  \scaleto{%
    \scalerel*[\widthof{\ensuremath{#1}}]{\kern-.6pt\bigwedge\kern-.6pt}%
    {\rule[-\textheight/2]{1ex}{\textheight}}
  }{\textheight}%
}{0.5ex}}%
\stackon[1pt]{#1}{\scalebox{-0.75}{\tmpbox}}%
}

\makeatletter
\renewenvironment{proof}[1][\proofname]{%
  \par\pushQED{\qed}\normalfont%
  \topsep6\p@\@plus6\p@\relax
  \trivlist\item[\hskip\labelsep\bfseries#1\@addpunct{:}]%
  \ignorespaces
}{%
  \popQED\endtrivlist\@endpefalse
}
\makeatother

\usepackage{algorithm}
\usepackage[algo2e, vlined, ruled]{algorithm2e}
\usepackage{algorithmicx}

\makeatletter
\newcommand{\eqnum}{\leavevmode\hfill\refstepcounter{equation}\textup{\tagform@{\theequation}}}
\makeatother
\newcounter{protocol}


\newtheoremstyle{mytheoremstyle} 
        {\topsep}                    
        {\topsep}                    
        {\slshape\fontfamily{lmr}\selectfont}                   
        {}                           
        {\bfseries\fontfamily{lmr}\selectfont}                   
        {:}                          
        {.5em}                       
        {}  
\theoremstyle{mytheoremstyle}
\newtheorem{thm}{Theorem}
\newtheorem{prop}[thm]{Proposition}
\newtheorem{cor}[thm]{Corollary}
\newtheorem{lem}[thm]{Lemma}

\newtheorem{dff}[thm]{Definition}

\usepackage{prettyref}
\newcommand{\pref}[1]{\prettyref{#1}}
\newcommand{\savehyperref}[2]{\texorpdfstring{\hyperref[#1]{#2}}{#2}}
\newrefformat{thm}{\savehyperref{#1}{Theorem~\ref*{#1}}}
\newrefformat{prop}{\savehyperref{#1}{Proposition~\ref*{#1}}}
\newrefformat{cor}{\savehyperref{#1}{Corollary~\ref*{#1}}}
\newrefformat{lem}{\savehyperref{#1}{Lemma~\ref*{#1}}}
\newrefformat{fact}{\savehyperref{#1}{Fact~\ref*{#1}}}
\newrefformat{conj}{\savehyperref{#1}{Conjecture~\ref*{#1}}}
\newrefformat{ass}{\savehyperref{#1}{Assumption~\ref*{#1}}}
\newrefformat{ppt}{\savehyperref{#1}{Property~\ref*{#1}}}
\newrefformat{dff}{\savehyperref{#1}{Definition~\ref*{#1}}}
\newrefformat{cond}{\savehyperref{#1}{Condition~\ref*{#1}}}
\newrefformat{ex}{\savehyperref{#1}{Example~\ref*{#1}}}
\newrefformat{rmk}{\savehyperref{#1}{Remark~\ref*{#1}}}
\newrefformat{alg}{\savehyperref{#1}{Algorithm~\ref*{#1}}}
\newrefformat{ptc}{\savehyperref{#1}{Protocol~\ref*{#1}}}
\newrefformat{sec}{\savehyperref{#1}{Section~\ref*{#1}}}
\newrefformat{app}{\savehyperref{#1}{Appendix~\ref*{#1}}}
\newrefformat{fig}{\savehyperref{#1}{Figure~\ref*{#1}}}
\newrefformat{tab}{\savehyperref{#1}{Table~\ref*{#1}}}

\makeatletter
\def\cleardoublepage{\clearpage\if@twoside \ifodd\c@page\else
\hbox{}
\thispagestyle{empty}
\newpage
\if@twocolumn\hbox{}\newpage\fi\fi\fi}
\makeatother 

\makeatletter
\DeclareTextCommand{\textprime}{\encodingdefault}{%
  \mbox{$\m@th'\kern-\scriptspace$}%
}
\makeatother

  \def\Ga{\Gamma} \def\Dl{\Delta}
   
\def\Lm{\Lambda}  \def\Om{\Omega}
  \def\ga{\gamma} \def\dl{\delta}
\def\ep{\epsilon}   
\def\lm{\lambda}

 \def\vth{\vartheta} \def\vphi{\varphi}
\def\Ac{\mathcal{A}}   
   \def\Hc{\mathcal{H}}
  \def\Kc{\mathcal{K}} 
  \def\Oc{\mathcal{O}} 
  \def\Sc{\mathcal{S}} 
   \def\Xc{\mathcal{X}}
 
\def\E{\mathbb{E}}   
\def\I{\mathbb{I}}   
   \def\P{\mathbb{P}}
 \def\R{\mathbb{R}}

  \def\Cf{\mathfrak{C}} 
   
 \def\Rb{\bar{R}}

   \def\hb{\bar{h}}

   \def\tb{\bar{t}}

  \def\Oh{\hat{O}} 
\def\Qh{\widehat{Q}} \def\Rh{\hat{R}}  
 \def\Vh{\widehat{V}}  
 \def\Zh{\widehat{Z}}

 \def\zh{\widehat{z}}

\def\pit{\tilde{\pi}}

\usepackage{stackengine}
\stackMath
\newcommand\tsup[2][2]{%
 \def\useanchorwidth{T}%
  \ifnum#1>1%
    \stackon[-.5pt]{\tsup[\numexpr#1-1\relax]{#2}}{\scriptscriptstyle\sim}%
  \else%
    \stackon[.5pt]{#2}{\scriptscriptstyle\sim}%
  \fi%
}

\usepackage{accents}
\newlength{\dhatheight}

\newcommand*{\bigdot}[1]{%
  \accentset{\mbox{\LARGE\bfseries .}}{#1}}
  
\def\vthh{\widehat{\vth}}
\def\phih{\widehat{\phi}}

\def\lmh{\widehat{\lm}}
\def\Lmh{\widehat{\Lm}}

\def\Pit{\tsup[1]{\Pi}}

\def\pit{\tsup[1]{\pi}}
\def\etat{\tsup[1]{\eta}}

\def\Ht{\tsup[1]{H}}
\def\Vt{\tsup[1]{V}}
\def\Qt{\tsup[1]{Q}}

\def\dt{\tsup[1]{d}}

\def\sm{\setminus}

\def\es{\enspace}
\def\ra{\rightarrow}

\def\1{\mathbf{1}}
\def\0{\mathbf{0}}
\def\eqsp{{\hspace{1.35em}}}
\DeclareMathOperator*{\argmax}{argmax} 

\newcommand{\rn}[1]{\textup{\lowercase\expandafter{\romannumeral#1}}}

\def\bms{\begin{bmatrix}}
\def\bme{\end{bmatrix}}
\def\beq{\begin{equation}}
\def\eeq{\end{equation}}
\def\bal{\begin{equation}\begin{aligned}}
\def\eal{\end{aligned}\end{equation}}
\def\bals{\begin{equation*}\begin{aligned}}
\def\eals{\end{aligned}\end{equation*}}

\begin{document}

\title{Common Information based Approximate State Representations in Multi-Agent Reinforcement Learning}
\author{
Hsu Kao\\
University of Michigan\\
\texttt{hsukao@umich.edu}\\
\And
Vijay Subramanian\\
University of Michigan\\
\texttt{vgsubram@umich.edu}\\
}
\maketitle

\begin{abstract}
Due to information asymmetry, finding optimal policies for Decentralized Partially Observable Markov Decision Processes (Dec-POMDPs) is hard with the complexity growing doubly exponentially in the horizon length. The challenge increases greatly in the multi-agent reinforcement learning (MARL) setting where the transition probabilities, observation kernel, and reward function are unknown. Here, we develop a general  compression framework with approximate common and private state representations, based on which decentralized policies can be constructed.  We derive the optimality gap of executing dynamic programming (DP) with the approximate states in terms of the approximation error parameters and the remaining time steps. When the compression is exact (no error), the resulting DP is equivalent to the one in existing work. Our general framework generalizes a number of methods proposed in the literature. The results shed light on designing practically useful deep-MARL network structures under the ``centralized learning distributed execution" scheme.
\end{abstract}

\section{\uppercase{Introduction}}
Finding optimal policies for Decentralized Partially Observable Markov Decision Processes (Dec-POMDPs) is hard due to \emph{information asymmetry}, which refers to the mismatch in the set of information each agent has in a multi-agent environment. In fact, a finite-horizon Dec-POMDP with more than one agent is NEXP-complete \citep{Bernstein_DecPOMDP_2002}, implying a doubly exponential complexity growth in the horizon length. In decentralized control theory, theoretical solutions have been proposed to find the optimal control laws for Dec-POMDPs. Notably among them is the common information (CI) approach \citep{Demos_CI_2013}, a framework that decomposes the decision of a full policy into the decision of a ``prescription policy" from the CI known by all the agents, and the ``prescription" itself which is a full characterization of how the agents should act based on any realization of their own private information (PI). This approach effectively transforms the decentralized model back to a centralized one from the view of a fictitious ``coordinator" who only observes the CI, and permits a coordinator level sequential decomposition using a belief state a'la POMDPs \citep{Kumar_StochCtrl_2015}.

The challenge increases greatly in the multi-agent reinforcement learning (MARL) setting where the \emph{model} -- transition probabilities, observation kernel, and reward function -- is unknown. When the agents learn concurrently, information asymmetry causes another issue called the ``non-stationarity issue," since the effective environment observed by each agent is time-varying as the other agents learn and update their policies. The issue can be alleviated in principle by the ``centralized learning and distributed execution" scheme \citep{Dibangoye_CentralLearn_2018} as the learning is from the coordinator's viewpoint; indeed, if agents only update their policies using CI, they can perfectly track others' policies. However, there is still a big gap in applying the CI approach to the MARL setting. First, the Bayesian updates of the belief state in the CI approach require the knowledge of the model, which are not available in the MARL setting. Moreover, the linear growth of length of private histories leads to the doubly exponential growth of the space of prescriptions in time, which is explosively large even for toy-size environments and forbids any practical explorations in such space. One natural question is whether we can restrict attention to some policies (and prescriptions) that take some state variable as inputs without losing much performance, where the state variables encapsulate the crucial information relevant to future decisions in a time-invariant domain, and where the representations (ways of encapsulation) can be learned without the knowledge of the model.

In this paper, we formulate good approximate common and private state representations for learning close-to-optimal policies in unknown finite-horizon Dec-POMDPs, where each agent receives its own private information plus a \emph{common} observation. The agents also share the same commonly observed rewards; however, they may not know each others' actions. We propose conditions in \pref{dff:ASPS} for an approximate sufficient private state (ASPS), which compresses an agent's private information, i.e., its action observation history (AOH), and conditions in \pref{dff:ASCS} for an approximate sufficient common state (ASCS), which compresses the fictitious coordinator's AOH, with the actions being ASPS-based prescriptions and the observations being common observations. Critically, using \pref{thm:ASPS} and \pref{thm:ASCS}, in \pref{thm:ASCS+ASPS} we derive the optimality gap in terms of the error parameters of our compression and the remaining time steps, between the values of two dynamic programmings (DPs): one in \pref{alg:FCS-FPS} for the optimal policy using the CI approach without compression, with states being the complete coordinator's AOHs and actions being the prescriptions from \cite{Demos_CI_2013}; and the other in \pref{alg:ASCS-ASPS} using our framework with states being any \emph{valid}\footnote{Satisfying the approximation criteria.} ASCSs and the actions being ASPS-based prescriptions for valid ASPS. Our framework generalizes a number of results in the literature: first, it extends the approximate information state (AIS) framework \citep{Aditya_AIS_2019,Aditya_AIS_2020} to the multi-agent setting; second, it extends the CI approach \citep{Demos_CI_2013} and the follow-up sufficient private information (SPI) framework that compresses the private states \citep{Hamid_SPI_2021}, to their general approximate state representation counterpart; third, it generalizes the work by \cite{Weichao_ASPI_2020} to include \emph{non-injective} compressions and a general approximate common state representation. Our results can provide guidance on designing Deep Learning (DL) structures to learn the (compressed) state representations and the optimal policies (using learned representations) under the centralized learning distributed execution scheme, which applies to practical offline or online MARL settings.

\vspace{-10pt}
\paragraph{Related Work.}
The problem of state representation is well studied in the single-agent POMDP case. Stochastic control theory details the conditions an \emph{information state} (IS) needs to satisfy so that it acts as the Markov state in an equivalent MDP so one may only consider IS-based policies without loss of generality \citep{Aditya_DecCtrl_2016}; the belief state is an example of such IS \citep{Kumar_StochCtrl_2015}. \cite{Aditya_AIS_2020} extends the idea to an \emph{approximate information state} (AIS), where the IS conditions  hold approximately; importantly, the optimality gap of running DP with any valid AIS is quantified. Based on their AIS scheme, they propose a DL framework that learns the AIS representation without knowing the model. Recent work on \emph{Deep Bisimulation for Control} (DBC) \citep{Amy_AIS_2021} in the DL literature uses similar ideas: they train an encoder to predict well the instantaneous rewards and transitions, and use the encoder output to train the policies. The encoder is an encapsulation or a compression. The optimality gap established is similar to the result of the infinite horizon case in \cite{Aditya_AIS_2019}. There are more representation learning schemes not requiring model knowledge in the DL or RL literature, e.g. \cite{Ha_AIS_2018}, with the bulk without theoretical guidance or guarantees.

In the multi-agent context, \cite{Demos_CI_2013} propose a belief IS for the coordinator using the CI approach, without compressing agents' private information. \cite{Hamid_SPI_2018} further compress private histories to \emph{sufficient private information} (SPI) so that the corresponding spaces of the belief IS and prescriptions are time-invariant. They identify conditions such that restricting attention to SPI-based policies is without loss of optimality. However, not only do they consider a control setting where the model is required, but also only present compression of the common history to a belief state, which is a narrow class of compression schemes. Nevertheless, this work will be a starting point of our work.  \cite{Weichao_ASPI_2020} consider an information state embedding that injectively maps agents' histories to representations in a fixed domain, and quantify the effect of the embedding on the value function like \cite{Aditya_AIS_2019}. However, their requirement that the mapping is injective is impractical for two reasons: one, an injective mapping does not reduce the policy complexity; and two, real world applications often demand non-injective encapsulations - e.g., tiger~\citep{Kaelbling_POMDP_1998} where one IS is the number of right observations minus the number of left observations, which is non-injective. Moreover, they also compress the common state to a belief state, but it is unclear how this can be done in practice without model information.

Another line of work in deep-MARL literature also applies the notion of CI (also known as the common knowledge) to solving MARL problems \citep{Schroeder_ComKnow_2019,Foerster_Hanabi1_2019,Foerster_Hanabi2_2019,Sokota_CAPI_2021}. They search for optimal policies for a Dec-POMDP when the model is known, while we consider designing sample efficient and lower regret learning algorithms in an offline or online MARL setting for an unknown model. Moreover, many of them involve heuristic or approximation methods without knowing the potential loss from the approximations or apply a variety of machine learning schemes without a theoretical basis or understanding.

\section{\uppercase{Preliminaries}}
\paragraph{Notation.}
Let $\Dl(\Xc)$ denote the set of distributions on the space $\Xc$, and $\Om(X)$ denote the space where the variable $X$ takes values. Superscripts are used as the agent index and subscripts as the time index. The notation $X_{a:b}^{c:d}$ denotes the tuple $(X_a^c,\dots,X_a^d,\dots,X_b^c,\dots,X_b^d)$. In some cases superscripts or subscripts are omitted, and if so the meaning will be clarified. Capital letters are used for random variables while lower case letters are for their realizations. For random variables $X$ with a realization $x$, we use the short hand notation $\P(\cdot|x)\triangleq\P(\cdot|X=x)$ and $\E[\cdot|x]\triangleq\E[\cdot|X=x]$. If a random variable appears without realization in a place other than the operand of $\E$, then it means the related equation should hold for any Borel measurable subset in its domain.

\subsection{Dec-POMDP Model}

Suppose there are $N$ agents in the system. We consider the Dec-POMDP model, i.e., a tuple $(\Sc,\Ac,\P_T,\mathcal{R},\Oc,\P_O,T,\P_I)$ where the quantities are: $\Sc$ is the state space; $\Ac=\Ac^1\times\cdots\times\Ac^N$ is the joint action space whose elements are joint actions $A=(A^1,\dots,A^N)$; $\P_T:\Sc\times\Ac\ra\Dl(\Sc)$ is the transition kernel mapping a current state and a joint action to a distribution of new states $\P_T:\Sc\times\Ac\ra\Dl(\Sc)$; $\mathcal{R}:\Sc\times\Ac\ra\Dl(\R)$ is the reward function mapping a current state and a joint action to a probability distribution on the reals; $\Oc=\Oc^0\times\Oc^1\times\cdots\times\Oc^N$ is the joint observation space whose elements are joint observations $O=(O^0,O^1,\dots,O^N)$, where $O^0$ is commonly observed but $O^n$ is only observed by agent $n$; $\P_T:\Sc\ra\Dl(\Oc)$ is the observation kernel mapping a current state to a distribution of joint observations; $T$ is the time horizon; $\P_I\in\Dl(\Sc)$ is the initial state distribution. In comparison to the standard Dec-POMDP model \citep{DecPOMDPBook_2016}, we have an additional common observation (including the reward), and our observations depend only on the current state.

We assume $\Sc$, $\Ac$, $\Oc$, and $T$ are finite and known in advance, while $\P_T$, $R$, $\P_O$, and $\P_I$ are unknown in the MARL setting. Further, agents have perfect recall. At time $t$, agent $n$ observes $(O_t^0,O_t^n)$ generated from $\P_O(S_t)$, then uses the policy $A_t^n=g_t^n(O_{1:t}^0,g_{1:t-1},H_t^n)$ to select its action, where $g_s=g_s^{1:N}$ and $H_t^n=(A_{1:t-1}^n,O_{1:t}^n)$ is agent $n$'s private history and known as its AOH. The agents receive a reward $R_t\triangleq R(S_t,A_t)$ sampled from $\mathcal{R}(S_t,A_t)$, and the next state $S_{t+1}$ is generated from $\P_T(S_t,A_t)$. The goal is to find a policy $g=g_{1:T}$ to maximize the common cumulative reward 
\beq
\E\left[\sum_{t=1}^TR(S_t,A_t)\bigg|g\right],
\eeq
where the expectation is taken over the measure generated by policy $g$ applied to model $(\P_T,\mathcal{R},\P_O,\P_I)$.

\subsection{AIS Framework}\label{sec:AIS}
In the single-agent POMDP setting, the spaces $\Ac$ and $\Oc$ are not product spaces, and at time $t$ the agent's policy is of the form $g_t:\Om(H_t)\ra\Om(A_t)$, where $H_t=(A_{1:t-1},O_{1:t})$ is the agent's AOH. Note the policy space grows exponentially in $t$ as the length of $H_t$ grows linearly in $t$. \cite{Aditya_AIS_2019} give conditions of a representation encapsulating the information in $H_t$ that is approximately sufficient for decision purposes into a time-invariant space.

\begin{dff}\label{dff:AIS}
An $(\ep,\dl)$-approximate information state $\Zh_t$ is the output of a function $\Zh_t=\vthh_t(H_t)$ that satisfies the following properties:
\begin{enumerate}[label=\textbf{(AIS\arabic*)},leftmargin=*,topsep=-5pt,itemsep=0pt]
\setlength{\parskip}{0pt}
\item It evolves recursively $\Zh_{t+1}=\phih_t(\Zh_t,A_t,O_{t+1})$.
\item It suffices for approximate performance evaluation $|\E[R_t|h_t,a_t]-\E[R_t|\zh_t,a_t]|\leq\ep$ $\forall\es h_t,a_t$.
\item It suffices for approximately predicting the observation, i.e., $\forall\es h_t,a_t$, we have $\Kc(\P(O_{t+1}|h_t,a_t),\P(O_{t+1}|\zh_t,a_t))\leq\dl$ , where $\Kc(\cdot,\cdot)$ is a distance between two distributions\footnote{For example, Wasserstein and total variation distances.}.
\end{enumerate}
\end{dff}

The value function at $t$ obtained from Bellman equations with $\Zh$'s as states falls behind the optimal value function at the most by an expression linear in $T-t$, $\ep$, and $\dl$ \citep{Aditya_AIS_2019}. When $\ep=\dl=0$, the expression is $0$, and the AIS $\Zh$ degenerates to an IS $Z$.

In \cite{Aditya_AIS_2019}, a DL framework is provided to find an ``approximate mapping" $\vthh_t(\cdot)$ for any given POMDP model. The idea is to interpret the quantities in the LHS of (AIS2) and (AIS3) as driving the learning loss in DL, and let existing DL optimization algorithms find good mappings. The resulting AIS can then be used as the state in common policy approximation methods to find a near-optimal policy.

\subsection{Common Information based DPs}

\subsubsection{DP with No Compression}
In a DecPOMDP, the action decision for agent $n$ at time $t$,  $A_t^n=g_t^n(O_{1:t}^0,g_{1:t-1},H_t^n)$, can be split into two steps. In the first step, based on past common observations and policies $(O_{1:t}^0,g_{1:t-1})$ (using perfect recall), the agent decides $g_t^n$ and hence $\Ga_t^n(\cdot)\triangleq g_t^n(O_{1:t}^0,g_{1:t-1},\cdot)$; then in the second step, it simply applies $\Ga_t^n$ to $H_t^n$ to obtain the action $A_t^n=\Ga_t^n(H_t^n)$. The function $\Ga_t^n$ is called the prescription (function), since it \emph{prescribes} what the agent should do based on any possible realization of its private information.

This decomposition technique is called the CI approach \citep{Demos_CI_2013}. Note that the actual decision is carried out in the first step and solely upon CI (perfect recall makes policy common knowledge). One may then imagine there is a fictitious coordinator, labelled agent $0$. At time $t$, the coordinator's policy is of the form $d_t:\Om(H_t^0)\ra\Om(\Ga_t)$, where $H_t^0\triangleq(O_{1:t}^0,\Ga_{1:t-1})$ is equivalent to $(O_{1:t}^0,g_{1:t-1})$ and $\Ga_t=\Ga_t^{1:N}$; then it sends $\Ga_t$ to every agent, and agent $n$ selects $A_t^n=\Ga_t^n(H_t^n)$. It is shown that this decomposition is without loss of generality (so without loss of optimality too). The coordinator observes common observation $O_t^0$ and chooses action $\Ga_t$; hence, $H_t^0$ can be seen as the coordinator's AOH and will be called the full common state (FCS), while $H_t^n$ will be referred to as the full private state (FPS) of agent $n$. From the perspective of the coordinator, the problem is now a centralized POMDP, and the goal is to find a policy $d=d_{1:T}$ that maximizes the expected cumulative reward. This permits a sequential decomposition with FCS as the state and an FPS-based prescription (meaning the prescription takes FPS as its input) as the action, which is presented in \pref{alg:FCS-FPS}.

\begin{algorithm}
\caption{Dynamic Programming with FCSs and FPS-based Prescriptions}\label{alg:FCS-FPS}
$V_{T+1}(h_{T+1}^0)\triangleq0$\\
\For{$t=T,\dots,1$}{
$Q_t(h_t^0,\ga_t)=\E\big[R(S_t,\Ga_t(H_t^{1:N}))+$\\
\hspace*{20pt}$V_{t+1}((H_t^0,\Ga_t,O_{t+1}^0))|H_t^0=h_t^0,\Ga_t=\ga_t\big]$\\
$V_t(h_t^0)=\max_{\ga_t\in\Om(\Ga_t)}Q_t(h_t^0,\ga_t)$
}
\end{algorithm}

In practice, the coordinator is virtual and the computation of the coordinator is carried out in all agents -- this is viable since the coordinator's computation only requires CI, which every agent has access to. Note the update of the state is done by direct concatenation of the incoming $\Ga_t$ and $O_{t+1}^0$.

\subsubsection{DP with BCS}
\cite{Demos_CI_2013} further compresses the FCS to the belief common state (BCS) $\Pi_t=\P(S_t,H_t^{1:N}|H_t^0)$, which is the conditional distribution on the state and the FPSs given the FCS. It is shown that restricting attention to coordinator's policy of the form $\bigdot{d}_t:\Om(\Pi_t)\ra\Om(\Ga_t)$ is without loss of optimality. The DP presented thus uses this BCS as the state and an FPS-based prescription as the action -- see \pref{app:BCS}.

There are two problems with this approach when applied to the MARL setting. First, the BCS is updated via a Bayesian update using $\P_T$ and $\P_O$, which requires model knowledge. Second, the growing length of $H_t^{1:N}$ makes the spaces of $\Pi_t$ and $\Ga_t$ explosively large and impossible to explore. However, at a conceptual level we can apply the AIS framework to the centralized POMDP of the coordinator\footnote{Strictly speaking, this requires a straightforward extension to time-varying action spaces for different time steps -- see \cite{Aditya_AIS_2020} Section 5 for details.}; the underlying decentralized information structure coupled with increasing domain of private information makes practical implementations of this scheme challenging. 

\subsubsection{DP with BCS and SPI}
To alleviate the aforementioned dimensionality issue, \cite{Hamid_SPI_2018} further compresses the FPS to a representation called the sufficient private information (SPI) lying in a time-invariant domain. They identify a set of conditions for the compression so that the SPI is sufficient for decision making purposes.

\begin{dff}\label{dff:SPI}
A sufficient private information (SPI) $Z_t^{1:N}$ is a tuple of outputs of a set of functions $Z_t^n=\vth_t^n(H_t^0,H_t^n)$ $\forall\es n\in[N]$ satisfying the properties:
\begin{enumerate}[label=\textbf{(SPI\arabic*)},leftmargin=*,topsep=-5pt,itemsep=0pt]
\setlength{\parskip}{0pt}
\item It evolves recursively, i.e., $\forall\es n\in[N]$, $Z_{t+1}^n=\phi_t^n(Z_{t-1}^n,\Ga_{t-1},O_t^0,A_{t-1}^n,O_t^n,g_{1:t-1})$.
\item It suffices for performance evaluation \scalebox{1}{$\E[R(S_t,A_t)|h_t^0,h_t^n,a_t]=\E[R(S_t,A_t)|h_t^0,z_t^n,a_t]$} $\forall\es h_t^0,h_t^n,a_t$.
\item It suffices for predicting itself and the common observation $\P(Z_{t+1}^{1:N},O_{t+1}^0|h_t^0,h_t^{1:N},\ga_t,a_t)$ $=\P(Z_{t+1}^{1:N},O_{t+1}^0|h_t^0,z_t^{1:N},\ga_t,a_t)$ $\forall\es h_t^{0:N},\ga_t,a_t$.
\item It suffices for predicting other agents' SPI $\P(Z_t^{-n}|h_t^0,h_t^n)=\P(Z_t^{-n}|h_t^0,z_t^n)$ $\forall\es h_t^0,h_t^n$.
\end{enumerate}
\end{dff}

The coordinator now considers SPI-based prescriptions $\Lm_t=\Lm_t^{1:N}$ where $A_t^n=\Lm_t^n(Z_t^n)$, and the BCS is changed to $\Pit_t=\P(S_t,Z_t^{1:N}|\Ht_t^0)$ where $\Ht_t^0=(O_1^0,\Lm_1,\dots,O_t^0)$. It is shown that restricting attention to coordinator's policy of the form $\dt_t:\Om(\Pit_t)\ra\Om(\Lm_t)$ is without loss of optimality. The resulting DP uses the BCS as the state and SPI-based prescription as the action -- see \pref{app:BCS+SPI}. Note that the compression actually leads to an action compression for the coordinator -- from FPS-based prescriptions to SPI-based prescriptions -- which has no loss in performance.

With $\Pit_t$, $Z_t^{1:N}$, and $\Lm_t$ all lying in time-invariant spaces, the complexity no longer grows with time. However, it is unclear how to find mappings satisfying \pref{dff:SPI} and update the BCSs in an MARL setting. Further, the solution focuses on a decentralized setting wherein the (lossless) compression functions are consistent (common knowledge), and the performance assessments and predictions are based only on the information of any particular agent. Ensuring these properties in the RL context would require significant communication, particularly during training.

\section{\uppercase{Approximate State Representations}}\label{sec:main-result}

We seek to extend the idea of identifying representations sufficient for approximately optimal decision making from \pref{sec:AIS} to the multi-agent setting, and develop a general compression framework for common states and private states (hence also prescriptions) whose mappings can be learned from samples obtained by interacting with the environment alone.

In this section, we propose our general states representation framework for approximate planning and control in partially observable MARL problems. We start by compressing private histories to ASPS; for the coordinator, this induces an action compression from FPS-based prescriptions to ASPS-based prescriptions. Then based on this compression, the common history is further compressed to ASCS.

The framework we develop will be consistent with the philosophy of recent empirical MARL work wherein there is a centralized agent called the supervisor. The supervisor observes all the quantities and develops good compression of private information and common information that the coordinator can use to produce close-to-optimal prescriptions (using the compressed common information), which can be implemented by the agents using just their own compressed private information. We detail the supervisor in \pref{sec:supervisor} but point out here that it has the knowledge of $H_t^{0:N}$ for all $t\in[T]$. Note that this viewpoint is consistent with the ``centralized training with distributed execution" setting of the empirical MARL work.

\subsection{Compressing Private States}

\begin{dff}\label{dff:ASPS}
An $(\ep_p,\dl_p)$-approximate sufficient private state (ASPS) $\Zh_t^{1:N}$ is a tuple of outputs of a set of functions $\Zh_t^n=\vthh_t^n(H_t^0,H_t^n)$ $\forall\es n\in[N]$ satisfying:
\begin{enumerate}[label=\textbf{(ASPS\arabic*)},leftmargin=*,topsep=-5pt,itemsep=0pt]
\setlength{\parskip}{0pt}
\item It evolves in a recursive manner, that is, $\forall\es n\in[N]$, $\Zh_t^n=\phih_t^n(\Zh_{t-1}^n,\Ga_{t-1},O_t^0,O_t^n)$.
\item It suffices for approximate performance evaluation $\big|\E[R(S_t,A_t)|h_t^0,h_t^{1:N},a_t]-\E[R(S_t,A_t)|h_t^0,\zh_t^{1:N},a_t]\big|\leq\ep_p/4$ $\forall\es h_t^{0:N},a_t$.
\item It suffices for approximately predicting observations $\Kc\big(\P(O_{t+1}^{0:N}|h_t^0,h_t^{1:N},a_t),$ $\P(O_{t+1}^{0:N}|h_t^0,\zh_t^{1:N},a_t)\big)\leq\dl_p/8$ $\forall\es h_t^{0:N},a_t$.
\end{enumerate}
\end{dff}

This definition induces the ASPS-based prescription, which is a mapping $\lmh_t:\Om(\Zh_t^{1:N})\ra\Om(A_t)$ that prescribes the action tuple for all ASPSs $A_t=\lmh_t(\Zh_t^{1:N})=(\lmh_t^1(\Zh_t^1),\dots,\lmh_t^N(\Zh_t^N))$ in a component-wise manner. One can run a DP with FCSs as states and ASPS-based prescriptions as actions -- see \pref{alg:FCS-ASPS}.

\begin{algorithm}
\caption{Dynamic Programming with FCSs and ASPS-based Prescriptions}\label{alg:FCS-ASPS}
$\Vh_{T+1}(h_{T+1}^0)\triangleq0$\\
\For{$t=T,\dots,1$}{
$\Qh_t(h_t^0,\lmh_t)=\E\big[R_t(S_t,\lmh_t(\Zh_t^{1:N}))$\\
\hspace*{20pt}$+\Vh_{t+1}((H_t^0,\Lmh_t,O_{t+1}^0))|H_t^0=h_t^0,\Lmh_t=\lmh_t\big]$\\
$\Vh_t(h_t^0)=\max_{\lmh_t\in\Om(\Lmh_t)}\Qh_t(h_t^0,\lmh_t)$
}
\end{algorithm}

The compression is characterized by functions $\vthh_t^{1:N}$. These functions also relate $\Om(\Ga_t)$ and $\Om(\Lmh_t)$ as ASPS-based prescriptions are a strict subset of FPS-based prescritions; this detail will be explained in \pref{sec:ASPS-proof}. For now we note that here the conditions we set for the action compression from $\Om(\Ga_t)$ to $\Om(\Lmh_t)$ are on the private states instead of defining an encapsulation directly on the actions (i.e., prescriptions); moreover, the compression may depend on the common state $h_t^0$ as well. Hence, this falls outside of the action compression scheme studied in \cite{Aditya_AIS_2020}. We bound the error between the value functions obtained from \pref{alg:FCS-ASPS} and the optimal value functions obtained from \pref{alg:FCS-FPS} in the following theorem proved in \pref{sec:ASPS-proof}.

\begin{thm}\label{thm:ASPS}
Assume the reward function $R$ is uniformly bounded by $\Rb$. For any $h_t^0\in\Om(H_t^0)$ and $\ga^*\in$ $\argmax_\ga Q_t(h_t^0,\ga)$, there exists a $\lmh\in\Om(\Lmh_t)$ such that
\beq
Q_t(h_t^0,\ga^*)-\Qh_t(h_t^0,\lmh)\leq\frac{\tb(\tb+1)}{2}(\ep_p+T\Rb\dl_p)+(\tb+1)\ep_p,
\eeq
\beq
V_t(h_t^0)-\Vh_t(h_t^0)\leq\frac{\tb(\tb+1)}{2}(\ep_p+T\Rb\dl_p)+(\tb+1)\ep_p,
\eeq
where $\tb=T-t$.
\end{thm}

\subsection{Compressing Common States}\label{sec:ASCS}
While restricting attention to ASPS-based prescriptions, we further compress the common history to an approximate representation by applying the state compression result of \cite{Aditya_AIS_2019}.

\begin{dff}\label{dff:ASCS}
An $(\ep_c,\dl_c)$-approximate sufficient common state (ASCS) $\Zh_t^0$ is the output of a function $\Zh_t^0=\vthh_t^0(H_t^0)$ satisfying the properties:
\begin{enumerate}[label=\textbf{(ASCS\arabic*)},leftmargin=*,topsep=-5pt,itemsep=0pt]
\setlength{\parskip}{0pt}
\item It evolves in a recursive manner, that is, $\Zh_t^0=\phih_t^0(\Zh_{t-1}^0,\Lmh_{t-1},O_t^0)$.
\item It suffices for approximate performance evaluation, i.e., $\forall\es h_t^0,\lmh_t$, we have $\big|\E[R_t(S_t,A_t)|h_t^0,\lmh_t]-\E[R_t(S_t,A_t)|\zh_t^0,\lmh_t]\big|\leq\ep_c$.
\item It suffices for approximately predicting common observation, i.e., $\forall\es h_t^0,\lmh_t$, we have $\Kc\big(\P(O_{t+1}^0|h_t^0,\lmh_t),\P(O_{t+1}^0|\zh_t^0,\lmh_t)\big)\leq\dl_c/2$.
\end{enumerate}
\end{dff}

In our proposed representation framework, agents compress the CI and PI to ASCS $\Zh_t^0$ and ASPS $\Zh_t^{1:N}$, which can be updated recursively using the incoming CI and PI. Agents use the same policy $\widecheck{d}_t^*:\Om(\Zh_t^0)\ra\Om(\Lmh_t)$ to decide the ASPS-based prescription $\Lmh_t$ from $\Zh_t^0$, then they apply $\Lmh_t$ to their own ASPS $\Zh_t^n$ to obtain the action $A_t^n$. Approximately optimal policies then result from the DP with ASCSs as states and ASPS-based prescriptions as actions -- see \pref{alg:ASCS-ASPS}.

\begin{algorithm}
\caption{Dynamic Programming with ASCSs and ASPS-based Prescriptions}\label{alg:ASCS-ASPS}
$\widecheck{V}_{T+1}(\zh_{T+1}^0)\triangleq0$\\
\For{$t=T,\dots,1$}{
$\widecheck{Q}_t(\zh_t^0,\lmh_t)=\E\big[R(S_t,\Lmh_t(\Zh_t^{1:N}))$\\
\hspace*{20pt}$+\widecheck{V}_{t+1}(\phih_t^0(\Zh_t^0,\Lmh_t,O_{t+1}^0))|\Zh_t^0=\zh_t^0,\Lmh_t=\lmh_t\big]$\\
$\widecheck{V}_t(\zh_t^0)=\max_{\lmh_t\in\Om(\Lmh_t)}\widecheck{Q}_t(\zh_t^0,\lmh_t)$
}
\end{algorithm}

From \pref{alg:FCS-ASPS} to \pref{alg:ASCS-ASPS}, only the states are further compressed, so a gap result bounding the difference between the two DPs holds, similar to the result in \cite{Aditya_AIS_2019}. See \pref{app:ASCS} for details.

\begin{thm}\label{thm:ASCS}
Assume the reward function $R$ is uniformly bounded by $\Rb$. For any $h_t^0\in\Om(H_t^0)$ and $\lmh\in\Om(\Lmh_t)$, with $\tb=T-t$, we have
\beq
\Qh_t(h_t^0,\lmh)-\widecheck{Q}_t(\vthh_t^0(h_t^0),\lmh)\leq\tb(\ep_c+T\Rb\dl_c)+\ep_c,
\eeq
\beq
\Vh_t(h_t^0)-\widecheck{V}_t(\vthh_t^0(h_t^0))\leq\tb(\ep_c+T\Rb\dl_c)+\ep_c.
\eeq
\end{thm}

\subsection{Main Result}
Our main result bounds the optimality gap of value functions obtained from performing DP with the general common and private representations satisfying the conditions in \pref{dff:ASPS} and \pref{dff:ASCS} as in \pref{alg:ASCS-ASPS}, in comparison to the optimal value functions computed from \pref{alg:FCS-FPS}.

\begin{thm}\label{thm:ASCS+ASPS}
Assume the reward function $R$ is uniformly bounded by $\Rb$. For any $h_t^0\in\Om(H_t^0)$ and $\ga^*\in$ $\argmax_\ga Q_t(h_t^0,\ga)$, there exists a $\lmh\in\Om(\Lmh_t)$ such that
\beq
Q_t(h_t^0,\ga^*)-\widecheck{Q}_t(\vthh_t^0(h_t^0),\lmh)\leq\frac{\tb(\tb+1)}{2}(\ep_p+T\Rb\dl_p)+(\tb+1)(\ep_c+\ep_p)+\tb T\Rb\dl_c,
\eeq
\beq
V_t(h_t^0)-\widecheck{V}_t(\vthh_t^0(h_t^0))\leq\frac{\tb(\tb+1)}{2}(\ep_p+T\Rb\dl_p)+(\tb+1)(\ep_c+\ep_p)+\tb T\Rb\dl_c,
\eeq
where $\tb=T-t$.
\end{thm}
\begin{proof}
Combine \pref{thm:ASPS} and \pref{thm:ASCS}.
\end{proof}

We observe that the action compression induced by private state compression leads to a gap quadratic in remaining time $\tb=T-t$ (\pref{thm:ASPS}), and common state compression causes a gap linear in remaining time. Also, note that the gap decreases to $0$ as $(\ep_c,\dl_c,\ep_p,\dl_p)$ go to $0$. Having developed this result, the remaining questions for learning using the sample data from the environment are: how to learn the compression mappings $\vthh_{1:T}^{0:N}$ with small error; and how to learn good policies with the compressed representations. See \pref{app:framework} for a proposed scheme to answer both these questions using DL methods.

\subsection{Comparisons to Existing Schemes}\label{sec:comparison}
\cite{Demos_CI_2013} and \cite{Hamid_SPI_2018} provide lossless (performance-wise) compression. We refer to ASPS and its corresponding conditions with $\ep_p=\dl_p=0$ as SPS; similarly, we refer to ASCS and its corresponding conditions with $\ep_c=\dl_c=0$ as SCS. Missing proofs in this subsection are in \pref{app:comparison}.

\paragraph{Relation to \cite{Demos_CI_2013}.}
The private history is not compressed in \cite{Demos_CI_2013}, so it is clearly a special case of SPS. The BCS proposed in \cite{Demos_CI_2013} is a special case of SCS as well.
\begin{prop}\label{prop:BCS=SCS}
The BCS $\Pi_t=\P(S_t,H_t^{1:N}|H_t^0)$ satisfies the conditions of an SCS in \pref{dff:ASCS} with $\ep_c=\dl_c=0$.
\end{prop}

\paragraph{Relation to \cite{Hamid_SPI_2018}.}
Our conditions of SPS and \cite{Hamid_SPI_2018}'s conditions of SPI both lead to performance sufficiency of the space of SPI-based (or SPS-based) prescriptions. The two sets of conditions are similar but not exactly the same. Condition (SPI1) corresponds to (SPS1); however, (SPS1) is stricter since we require \emph{policy-independent compression}, while \cite{Hamid_SPI_2018} allow policy-dependent compression. Condition (SPI3) ensures \emph{future sufficiency} as does (SPS3). Conditions (SPI2) and (SPI4) together ensure \emph{present sufficiency} as does (SPS3).

\begin{prop}\label{prop:SPI-relation-1}
(SPS1) and (SPS3) imply (SPI3).
\end{prop}

\begin{prop}\label{prop:SPI-relation-2}
(SPS2) and (SPI4) imply (SPI2).
\end{prop}

Restricting to SPS, their BCS $\Pit_t=\P(S_t,Z_t^{1:N}|H_t^0)$ is a special case of SCS as well, so a result identical to \pref{prop:BCS=SCS} holds with FPS changed to SPS.

\paragraph{Relation to \cite{Weichao_ASPI_2020}.}
Their private state embedding does not require a recursive update (ASPS1), but demands injective functions $\vthh_t^{1:N}$. With this additional assumption they show linearity of the optimality gap in remaining time. For the common state, the BCS they consider is a special case of our SCS, just as the BCS of \cite{Hamid_SPI_2018}.

\section{\uppercase{Optimality Gap Analysis}}\label{sec:analysis}
In this section, we outline the optimality gaps introduced in \pref{sec:main-result}; details are in \pref{app:ASPS}.

\subsection{Supervisor's Functions}\label{sec:supervisor}
For better exposition, we introduce another set of $Q$/$V$ functions from an omniscient \emph{supervisor}'s perspective, for the original decision problem. 
The supervisor can access the \emph{union} of the information of all agents: at time $t$ the supervisor knows $H_t^{0:N}$. In contrast, coordinator's information is the \emph{intersection} of the information of all agents: $H_t^0$. The supervisor, however, only observes what is happening, lets the coordinator decide all the policies and prescriptions, and implements the coordinator's policies. Let $d_{1:T}^*$ be a coordinator's optimal policy solved using \pref{alg:FCS-FPS}, i.e. $d_t^*(h_t^0)=\argmax_{\ga_t\in\Om(\Ga_t)}Q_t(h_t^0,\ga_t)$. Then the $Q$/$V$ functions defined in \pref{alg:FCS-FPS} can be rewritten as
\beq
Q_t(h_t^0,\ga_t)=\E\left[\sum_{\tau=t}^TR_\tau\bigg|h_t^0,\ga_t,d_{t+1:T}^*\right],
\eeq
\beq\label{eq:5.3}
V_t(h_t^0)=\E\left[\sum_{\tau=t}^TR_\tau\bigg|h_t^0,d_{t:T}^*\right].
\eeq
The supervisor's $Q$/$V$ functions use similar concepts, but with \emph{supervisor's states} and \emph{coordinator's policies}.

\begin{dff}\label{dff:supervisor}
For any $h_t^{0:N}\in\Om(H_t^{0:N})$, $\ga_t\in\Om(\Ga_t)$, define the supervisor's $Q$ function as
\beq
Q_t^S(h_t^0,h_t^{1:N},\ga_t)
\triangleq\E\left[\sum_{\tau=t}^TR_\tau\bigg|h_t^0,h_t^{1:N},\ga_t,d_{t+1:T}^*\right],
\eeq
and the supervisor's $V$ function as
\beq\label{eq:5.5}
V_t^S(h_t^0,h_t^{1:N})\triangleq\E\left[\sum_{\tau=t}^TR_\tau\bigg|h_t^0,h_t^{1:N},d_{t:T}^*\right]=Q_t^S(h_t^0,h_t^{1:N},\ga_t^*),
\eeq
where $\ga_t^*\in\argmax_{\ga_t\in\Om(\Ga_t)}Q_t(h_t^0,\ga_t)$\footnote{The supervisor's $Q$ function and $V$ function are only defined when the FPS $h_t^{1:N}$ is \emph{admissible} under $h_t^0$, i.e. $\P(h_t^{1:N}|h_t^0)>0$. Throughout the rest of the paper, we assume that only admissible FPSs are considered.}.
\end{dff}

Then the coordinator's $Q$/$V$ functions can be expressed as the expectation of supervisor's $Q$/$V$ functions taken over the conditional distribution on FPSs given the FCS:
\beq\label{eq:supervisor-Q-expansion}
Q_t(h_t^0,\ga_t)=\sum_{h_t^{1:N}}\P(h_t^{1:N}|h_t^0)Q_t^S(h_t^0,h_t^{1:N},\ga_t),
\eeq
\beq
V_t(h_t^0)=\sum_{h_t^{1:N}}\P(h_t^{1:N}|h_t^0)V_t^S(h_t^0,h_t^{1:N}).
\eeq

\subsection{Proof of \pref{thm:ASPS}}\label{sec:ASPS-proof}
We first determine the relationship between the space of FPS-based prescriptions $\Om(\Ga_t)$ and the space of ASPS-based prescriptions $\Om(\Lmh_t)$. Consider a fixed $h_t^0$. Since the compression mappings $\Zh_t^{1:N}=\vthh_t^{1:N}(H_t^0,H_t^{1:N})$ are functions, there could be multiple $h_t^{1:N}$'s that are mapped to the same $\zh_t^{1:N}$. A $\lmh_t\in\Om(\Lmh_t)$ can thus be thought of as a special element of $\Om(\Ga_t)$ that prescribes the same action for all the FPSs $h_t^{1:N}$'s mapped to the same ASPS $\zh_t^{1:N}$. Hence, we can construct an injective \emph{extension mapping} from $\Om(\Lmh_t)$ to $\Om(\Ga_t)$ in this sense.

\begin{dff}\label{dff:5.9}
For any $h_t^0\in\Om(H_t^0)$, define the extension mapping $\psi_t:\Om(\Lmh_t)\times\Om(H_t^0)\ra\Om(\Ga_t)$ as follows: for any $h_t^{1:N}$ and $\lmh_t$, $\ga_t=\psi_t(\lmh_t,h_t^0)$ will first compress $h_t^{1:N}$ to $\zh_t^{1:N}=\vthh_t^{1:N}(h_t^0,h_t^{1:N})$ (hence $\psi_t$ implicitly depends on $\vthh_t$), then choose the action according to $\lmh_t(\zh_t^{1:N})$. That is,
\[
\ga_t(h_t^{1:N})=\psi_t(\lmh_t,h_t^0)(h_t^{1:N})\triangleq\lmh_t(\vthh_t^{1:N}(h_t^0,h_t^{1:N})).
\]
Given the compression $\vthh_t^{1:N}$, $\psi_t$ is well-defined. Under this circumstance and with an abuse of notation, $\ga_t=\psi_t(\lmh_t,h_t^0)$ will be written as $\ga_{\lmh_t,h_t^0}$ when the considered compression is clear from the context and will be referred to as the $\ga_t$ extended from $\lmh_t$ under $h_t^0$.
\end{dff}

The following proposition says that for any FCS one can find an ASPS-based prescription whose extension nearly achieves the same $Q$-value as an optimal prescription, up to a gap linear in the remaining time $\tb$. This implies that it nearly suffices to consider the class of prescriptions extended from ASPS-based prescriptions for DP purposes.

\begin{prop}\label{prop:ASPS-main}
Assume the reward function $R$ is uniformly bounded by $\Rb$. For any $h_t^0\in\Om(H_t^0)$ and $\ga^*\in\argmax_\ga Q_t(h_t^0,\ga)$, there exists a $\lmh\in\Om(\Lmh_t)$ with
\beq
\left|Q_t(h_t^0,\ga^*)-Q_t(h_t^0,\ga_{\lmh,h_t^0})\right|\leq\tb(\ep_p+T\Rb\dl_p)+\ep_p,
\eeq
which leads to
\beq
\left|V_t(h_t^0)-\max_{\lmh\in\Om(\Lmh_t)}Q_t(h_t^0,\ga_{\lmh,h_t^0})\right|\leq\tb(\ep_p+T\Rb\dl_p)+\ep_p.
\eeq
\end{prop}

Before proving this critical proposition we need a few intermediate results. The first key lemma says that with the same supervisor's state, the supervisor's $Q$-values for two different prescriptions will be the same as long as they prescribe the same action for the given private information.

\begin{lem}\label{lem:two-prescription}
For any $h_t^0\in\Om(H_t^0)$ and $h\in\Om(H_t^{1:N})$, let $\ga_1,\ga_2\in\Om(\Ga_t)$ be two prescriptions that choose the same action on $h$, i.e. $\ga_1(h)=\ga_2(h)=a$; then
\beq
Q_t^S(h_t^0,h,\ga_1)=Q_t^S(h_t^0,h,\ga_2).
\eeq
\end{lem}

Next we show that the supervisor's $Q$-values will be nearly the same for two different FPSs that map to the same ASPS and a prescription that prescribes the same action on these two FPSs.

\begin{lem}\label{lem:two-ASPS}
Assume $R$ is uniformly bounded by $\Rb$. For any $h_t^0\in\Om(H_t^0)$, let $h_1,h_2\in\Om(H_t^{1:N})$ be two FPSs under $h_t^0$ that map to the same ASPS $\zh\in\Om(\Zh_t^{1:N})$, i.e. $\zh=\vthh_t^{1:N}(h_t^0,h_1)=\vthh_t^{1:N}(h_t^0,h_2)$, and let $\ga\in\Om(\Ga_t)$ be a prescription that chooses the same action on these two FPSs $\ga(h_1)=\ga(h_2)=a$. Then
\beq
\left|Q_t^S(h_t^0,h_1,\ga)-Q_t^S(h_t^0,h_2,\ga)\right|\leq\tb(\ep_p+T\Rb\dl_p)/2+\ep_p/2.
\eeq
\end{lem}

Using the above two lemmas, we show that the supervisor's $V$ function will differ little for two supervisor's states with the same compression of private states.

\begin{cor}\label{cor:ASPS}
Assume $R$ is uniformly bounded by $\Rb$. For any $h_t^0\in\Om(H_t^0)$, let $h_1,h_2\in\Om(H_t^{1:N})$ be two FPSs under $h_t^0$ that map to the same ASPS $\zh\in\Om(\Zh_t^{1:N})$, i.e. $\zh=\vthh_t(h_t^0,h_1)=\vthh_t(h_t^0,h_2)$, and let $\ga^*\in\argmax_\ga Q_t(h_t^0,\ga)$ be an optimal prescription. Then
\beq
\left|V_t^S(h_t^0,h_1)-V_t^S(h_t^0,h_2)\right|\triangleq\left|Q_t^S(h_t^0,h_1,\ga^*)-Q_t^S(h_t^0,h_2,\ga^*)\right|\leq\tb(\ep_p+T\Rb\dl_p)/2+\ep_p/2.
\eeq
\end{cor}

\begin{proof}[Proof of \pref{prop:ASPS-main}]
Given an optimal prescription $\ga^*\in\argmax_\ga Q_t(h_t^0,\ga)$, we will specifically construct a $\lmh\in\Om(\Lmh_t)$ that serves for the claim. For each $\zh\in\Om(\Zh_t^{1:N})$, define
\beq
\Hc_{\zh}=\left\{h\in\Om(H_t^{1:N}):\vthh_t^{1:N}(h_t^0,h)=\zh\right\}
\eeq
to be the class of $h$'s in $\Om(H_t^{1:N})$ that are compressed to $\zh$ under the considered compression $\vthh_t^{1:N}$. By the Axiom of Choice, for each $\zh\in\Om(\Zh_t^{1:N})$ there exists a representative of the class $\Hc_{\zh}$ coming from arbitrary choice function, which we denote as $\hb_{\zh}$. We then construct the $\lmh$ by $\lmh(\zh)=\ga^*(\hb_{\zh})$; the corresponding extension in $\Om(\Ga_t)$ will be
\beq
\ga_{\lmh,h_t^0}(h)=\ga^*\left(\hb_{\vthh_t^{1:N}(h_t^0,h)}\right)\quad\forall\es h\in\Om(H_t^{1:N}),
\eeq
that is, the prescription first compresses the input FPS and finds the representative of the corresponding compression class, then it mimics what the optimal prescription would have done with the representative. For any $h\in\Om(H_t^{1:N})$, we have
\begin{align*}
\left|Q_t^S(h_t^0,h,\ga^*)-Q_t^S(h_t^0,h,\ga_{\lmh,h_t^0})\right|
&\leq\left|Q_t^S(h_t^0,h,\ga^*)-Q_t^S(h_t^0,\hb_{\vthh_t^{1:N}(h_t^0,h)},\ga^*)\right|\\
&\quad+\left|Q_t^S(h_t^0,\hb_{\vthh_t^{1:N}(h_t^0,h)},\ga^*)-Q_t^S(h_t^0,\hb_{\vthh_t^{1:N}(h_t^0,h)},\ga_{\lmh,h_t^0})\right|\\
&\quad+\left|Q_t^S(h_t^0,\hb_{\vthh_t^{1:N}(h_t^0,h)},\ga_{\lmh,h_t^0})-Q_t^S(h_t^0,h,\ga_{\lmh,h_t^0})\right|\\
&\leq(T-t)(\ep_p+T\Rb\dl_p)+\ep_p,
\end{align*}
as the first term is bounded by $(T-t)(\ep_p+T\Rb\dl_p)/2+\ep_p/2$ due to \pref{cor:ASPS}, the second term is $0$ due to \pref{lem:two-prescription}, and the third term is bounded by $(T-t)(\ep_p+T\Rb\dl_p)/2+\ep_p/2$ due to \pref{lem:two-ASPS}. If it happens to be the case that $h=\hb_{\vthh_t^{1:N}(h_t^0,h)}$, i.e. $h$ is the representative, then the original term $|Q_t^S(h_t^0,h,\ga^*)-Q_t^S(h_t^0,h,\ga_{\lmh,h_t^0})|$ is $0$. Taking the conditional expectation on $h$ given $h_t^0$ gives the claim.
\end{proof}

\begin{proof}[Proof of \pref{thm:ASPS}]
There are three main quantities: $V_t(h_t^0)$ is the value obtained from executing optimal FPS-based prescriptions to the end, $\max_{\lmh\in\Om(\Lmh_t)}Q_t(h_t^0,\ga_{\lmh,h_t^0})$ is from executing the optimal ASPS-based prescription for step $t$ and then optimal FPS-based prescriptions afterwards to the end, and $\Vh_t(h_t^0)$ is from executing optimal ASPS-based prescriptions to the end. \pref{prop:ASPS-main} establishes that restricting to ASPS-based prescriptions in \emph{one step} incurs a gap (between $V_t(h_t^0)$ and $\max_{\lmh\in\Om(\Lmh_t)}Q_t(h_t^0,\ga_{\lmh,h_t^0})$) \emph{linear} in $T-t$. Using an induction argument to accumulate this gap in \emph{every step} from $T$ back to $t$ yields the gap (between $V_t(h_t^0)$ and $\Vh_t(h_t^0)$) to be \emph{quadratic} in $T-t$. See \pref{app:ASPS} for detailed derivations.
\end{proof}

\section{\uppercase{Conclusion}}\label{sec:conclusion}
In this paper, we developed a general approximate state representation framework for MARL problems in a Dec-POMDP setting. We bounded the optimality gap in terms of the approximation error parameters and the number of remaining time steps. The theory provides guidance on designing deep-MARL algorithms, which has great potential in practical uses. Future directions include: exploring DL methods for applications using our framework, designing a representation for prescriptions, designing fully decentralized MARL schemes by adding communication, and extensions to general-sum games.

\bibliography{References}
\bibliographystyle{plainnat}

\newpage
\appendix
\section{Supplementary Details}\label{app:supp}

\subsection{More Related Work}\label{app:more-relate}
In \cite{Yuksel_FiniteMemory_2020} consider a special type of AIS -- the $N$-memory, which contains the information from the last $N$ steps. Here, the compression function is fixed but in contrast to \cite{Aditya_AIS_2020}, the approximation error given each history need not be uniform. When the model is known, they provide conditions that bound the regret of $N$-memory policies (policies that depend on $N$-memory), and an algorithm that finds optimal policies within this class. The first algorithm to learn the optimal policies of POMDPs with sub-linear regret in an online setting is proposed in \cite{Mehdi_POMDPOnline_2021}. Using a posterior sampling-based scheme, the algorithm maintains the posterior distribution on the unknown parameters of the considered POMDP, and adopts the optimal policy with respect to a set of parameters sampled from the distribution in each episode. The posterior update in the algorithm, however, heavily relies on the knowledge of the observation kernel, which is usually unknown in RL settings.

State representation for control is studied extensively in the literature \citep{Lesort_StateRep_2018}. Early work on predictive state representation (PSR) of POMDPs \citep{Littman_PSR_2002} only focuses on the encapsulation of the histories and does not explore its system prediction ability. The bisimulation relation clusters MDP states with similar rewards and transitions, and a bisimulation metric convexly combines the errors of the rewards and the transitions between two states \citep{Ferns_Bisim_2011}. The difference of the value functions of two states can be upper-bounded by the metric. The causal state representation \citep{Amy_CausalState_2021} for POMDPs clusters the histories in the space of AOHs that will produce the same future dynamics. Using the observation history as the state, the considered POMDP can be transformed into an MDP, so that the results from the bisimulation literature can be applied.

\subsection{DP with BCS}\label{app:BCS}
\begin{algorithm}
\caption{Dynamic Programming with BCSs and FPS-based Prescriptions}
$\bigdot{V}_{T+1}(\pi_{T+1})\triangleq0$\\
\For{$t=T,\dots,1$}{
$\bigdot{Q}_t(\pi_t,\ga_t)=\E\left[R(S_t,\Ga_t(H_t^{1:N}))+\bigdot{V}_{t+1}(\eta_t(\Pi_t^0,\Ga_t,O_{t+1}^0))|\Pi_t=\pi_t,\Ga_t=\ga_t\right]$\\
$\bigdot{V}_t(\pi_t)=\max_{\ga_t\in\Om(\Ga_t)}\bigdot{Q}_t(\pi_t,\ga_t)$
}
\end{algorithm}

\subsection{DP with BCS and SPI}\label{app:BCS+SPI}
\begin{algorithm}
\caption{Dynamic Programming with BCSs and SPI-based Prescriptions}
$\Vt_{T+1}(\pit_{T+1})\triangleq0$\\
\For{$t=T,\dots,1$}{
$\Qt_t(\pit_t,\lm_t)=\E\left[R(S_t,\Lm_t(Z_t^{1:N}))+\Vt_{t+1}(\etat_t(\Pit_t,\Lm_t,O_{t+1}^0))|\Pit_t=\pit_t,\Lm_t=\lm_t\right]$\\
$\Vt_t(\pit_t)=\max_{\lm_t\in\Om(\Lm_t)}\Qt_t(\pit_t,\lm_t)$
}
\end{algorithm}

\section{Omitted Analysis in \pref{sec:ASPS-proof}}\label{app:ASPS}

The following lemma shows that given the FPS, the actions the chosen prescription chooses for other FPSs does not affect the next step statistics.

\begin{lem}\label{lem:gamma-to-a}
Let $h_t^0\in\Om(H_t^0)$, $h\in\Om(H_t^{1:N})$, $\ga\in\Om(\Ga_t)$, and $a=\ga(h)\in\Om(A_t)$. Then
\[
\P(S_{t+1},H_{t+1}^{1:N}|H_t^0=h_t^0,H_t^{1:N}=h,\Ga_t=\ga)=\P(S_{t+1},H_{t+1}^{1:N}|H_t^0=h_t^0,H_t^{1:N}=h,A_t=a).
\]
\end{lem}
\begin{proof}
We will omit specifying the original random variables when their realizations are given in the proof.
\allowdisplaybreaks
\begin{align*}
&\eqsp\P(S_{t+1},H_{t+1}^{1:N}|h_t^0,h,\ga)=\P(S_{t+1},H_{t+1}^{1:N}|h_t^0,h,\ga,a)\\
&=\sum_{s_t}\P(s_t|h_t^0,h,\ga,a)\cdot\P(S_{t+1},H_{t+1}^{1:N}|h_t^0,h,\ga,a,s_t)\\
&=\sum_{s_t}\P(s_t|h_t^0,h,a)\cdot\P(S_{t+1}|h_t^0,h,\ga,a,s_t)\cdot\P(H_{t+1}^{1:N}|h_t^0,h,\ga,a,s_t,S_{t+1})\tag{$\ga$ is after $s_t$}\\
&=\sum_{s_t}\P(s_t|h_t^0,h,a)\cdot\P(S_{t+1}|a,s_t)\cdot\P(H_{t+1}^{1:N}|h_t^0,h,\ga,a,s_t,S_{t+1})\tag{$\P_T$ specifies $S_{t+1}$ given $S_t$ and $A_t$}\\
&=\sum_{s_t}\P(s_t|h_t^0,h,a)\cdot\P(S_{t+1}|a,s_t)
\cdot\P(O_{t+1}^{1:N}|h_t^0,h,\ga,a,s_t,S_{t+1})\tag{$H_{t+1}^{1:N}=(H_t^{1:N},A_t,O_{t+1}^{1:N})$}\\
&=\sum_{s_t}\P(s_t|h_t^0,h,a)\cdot\P(S_{t+1}|a,s_t)\cdot\P(O_{t+1}^{1:N}|S_{t+1})\tag{$\P_O$ specifies $O_{t+1}^{1:N}$ given $S_{t+1}$}\\
&=\sum_{s_t}\P(s_t|h_t^0,h,a)\cdot\P(S_{t+1},H_{t+1}^{1:N}|h_t^0,h,a,s_t)\\
&=\P(S_{t+1},H_{t+1}^{1:N}|h_t^0,h,a).
\end{align*}
\end{proof}

\begin{proof}[Proof of \pref{lem:two-prescription}]
The proof for the instantaneous part is straightforward as $S_t$ is irrelevant to the choice of $\Ga_t$
\begin{align*}
&\eqsp\E\left[R_t(S_t,\Ga_t(H_t^{1:N}))|h_t^0,h,\ga_1\right]=\sum_{s_t}\P(s_t|h_t^0,h,\ga_1)R_t(s_t,\ga_1(h))\\
&=\sum_{s_t}\P(s_t|h_t^0,h,\ga_1(h))R_t(s_t,a)\\
&=\sum_{s_t}\P(s_t|h_t^0,h)R_t(s_t,a)\tag{$\ga_1$ and $\ga_2$ are exogenously given}\\
&=\E\left[R_t(S_t,\Ga_t(H_t^{1:N}))|h_t^0,h,\ga_2\right]\tag{by symmetry}.
\end{align*}
To show equality for the continuation part, we first define the following policy for all $\tau=t+1,\dots,T$:
\[
d^\prime_\tau(h_\tau^0)=\left\{
\begin{array}{ll}
d_\tau^*(h_t^0,\ga_1,h_{t+1:\tau}^0)&\text{if }h_\tau^0=(h_t^0,\ga_2,h_{t+1:\tau}^0)\quad\forall\es h_{t+1:\tau}^0,\\
d_\tau^*(h_\tau^0)&\text{otherwise,}
\end{array}\right.
\]
where $d_\tau^*$ is an optimal policy at time step $\tau$. Also, we have $h_{t+1:\tau}^0=o_{t+1}^0$ when $\tau=t+1$, and $h_{t+1:\tau}^0=(o_{t+1}^0,\ga_{t+1},\dots,o_\tau^0)$ when $\tau>t+1$, so that the entire $(h_t^0,\ga_1,h_{t+1:\tau}^0)\in\Om(H_\tau^0)$. This policy performs the optimal policy at all times, except when $\ga_2$ is chosen at time $t$, it will mimic what the optimal policy would have done if $\ga_1$ was chosen instead; owing to perfect recall, future prescriptions can depend on past ones. Then
\bals
&\eqsp\E\left[V_{t+1}^S(H_{t+1}^0,H_{t+1}^{1:N})|h_t^0,h,\ga_1\right]
=\E\left[\sum_{\tau=t+1}^TR_\tau(S_\tau,A_\tau)\bigg|h_t^0,h,\ga_1,d_{t+1:T}^*\right]\\
&=\E\left[\sum_{\tau=t+1}^TR_\tau(S_\tau,A_\tau)\bigg|h_t^0,h,\ga_1,d^\prime_{t+1:T}\right]
\overset{(*)}{=}\E\left[\sum_{\tau=t+1}^TR_\tau(S_\tau,A_\tau)\bigg|h_t^0,h,\ga_2,d^\prime_{t+1:T}\right]\\
&\leq\E\left[\sum_{\tau=t+1}^TR_\tau(S_\tau,A_\tau)\bigg|h_t^0,h,\ga_2,d_{t+1:T}^*\right]
=\E\left[V_{t+1}^S(H_{t+1}^0,H_{t+1}^{1:N})|h_t^0,h,\ga_2\right];
\eals
where the inequality holds as $d^\prime_{t+1:T}$ is not an optimal choice from the current history. By symmetry, the inequality implies that
\[
\E\left[V_{t+1}^S(H_{t+1}^0,H_{t+1}^{1:N})|h_t^0,h,\ga_1\right]=\E\left[V_{t+1}^S(H_{t+1}^0,H_{t+1}^{1:N})|h_t^0,h,\ga_2\right].
\]
The equality labeled by $(*)$ follows from the fact that under the policy $d^\prime_{t+1:T}$, choosing $\ga_1$ and $\ga_2$ will generate the exact future statistics. We will show this in the following. We first prove the following claim using mathematical induction.
\\\textbf{Claim:} for all $\tau=t+1,\dots,T$, we have
\[
\P(S_{t+1:\tau},O_{t+1:\tau}^{0:N},A_{t+1:\tau}|h_t^0,h,\ga_1,a,d^\prime_{t+1:T})=\P(S_{t+1:\tau},O_{t+1:\tau}^{0:N},A_{t+1:\tau}|h_t^0,h,\ga_2,a,d^\prime_{t+1:T}).
\]
\underline{Base case}: the claim holds for $\tau=t+1$.
\allowdisplaybreaks
\begin{align*}
&\eqsp\P(S_{t+1},O_{t+1}^{0:N},A_{t+1}|h_t^0,h,\ga_1,a,d^\prime_{t+1:T})\\
&=\sum_{s_t}\P(s_t|h_t^0,h,\ga_1,a,d^\prime_{t+1:T})\cdot\P(S_{t+1},O_{t+1}^{0:N},A_{t+1}|h_t^0,h,\ga_1,a,s_t,d^\prime_{t+1:T})\\
&=\sum_{s_t}\P(s_t|h_t^0,h)\cdot\P(S_{t+1},O_{t+1}^{0:N},A_{t+1}|h_t^0,h,\ga_1,a,s_t,d^\prime_{t+1})\tag{$s_t$ is independent of $a$ given $\Ga_t$}\\
&=\sum_{s_t}\P(s_t|h_t^0,h)\cdot\P(S_{t+1}|h_t^0,h,\ga_1,a,s_t,d^\prime_{t+1})\cdot\P(O_{t+1}^{0:N},A_{t+1}|h_t^0,h,\ga_1,a,s_t,d^\prime_{t+1},S_{t+1})\\
&=\sum_{s_t}\P(s_t|h_t^0,h)\cdot\P(S_{t+1}|s_t,a)\cdot\P(O_{t+1}^{0:N},A_{t+1}|h_t^0,h,\ga_1,a,s_t,d^\prime_{t+1},S_{t+1})\tag{$\P_T$ specifies $S_{t+1}$ given $S_t$ and $A_t$}\\
&=\sum_{s_t}\P(s_t|h_t^0,h)\cdot\P(S_{t+1}|s_t,a)\cdot\P(O_{t+1}^{0:N}|h_t^0,h,\ga_1,a,s_t,d^\prime_{t+1},S_{t+1})\\
&\qquad\cdot\P(A_{t+1}|h_t^0,h,\ga_1,a,s_t,d^\prime_{t+1},S_{t+1},O_{t+1}^{0:N})\\
&=\sum_{s_t}\P(s_t|h_t^0,h)\cdot\P(S_{t+1}|s_t,a)\cdot\P(O_{t+1}^{0:N}|S_{t+1})\cdot\P(A_{t+1}|h_t^0,h,\ga_1,a,s_t,d^\prime_{t+1},S_{t+1},O_{t+1}^{0:N})\tag{$\P_O$ specifies $O_{t+1}^{0:N}$ given $S_{t+1}$}\\
&=\sum_{s_t}\P(s_t|h_t^0,h)\cdot\P(S_{t+1}|s_t,a)\cdot\P(O_{t+1}^{0:N}|S_{t+1})\cdot\I\left\{A_{t+1}=d^\prime_{t+1}(h_0,\ga_1,O_{t+1}^0)(h,a,O_{t+1}^{1:N})\right\}\\
&=\sum_{s_t}\P(s_t|h_t^0,h)\cdot\P(S_{t+1}|s_t,a)\cdot\P(O_{t+1}^{0:N}|S_{t+1})\cdot\I\left\{A_{t+1}=d^\prime_{t+1}(h_0,\ga_2,O_{t+1}^0)(h,a,O_{t+1}^{1:N})\right\}\tag{definition of $d^\prime_{t+1}$}\\
&=\P(S_{t+1},O_{t+1}^{0:N},A_{t+1}|h_t^0,h,\ga_2,a,d^\prime_{t+1:T})\tag{symmetric argument}.
\end{align*}
\underline{Induction step}: assuming the claim holds for $\tau$, we show it holds for $\tau+1$ as well.
\begin{align*}
&\eqsp\P(S_{t+1:\tau+1},O_{t+1:\tau+1}^{0:N},A_{t+1:\tau+1}|h_t^0,h,\ga_1,a,d^\prime_{t+1:T})\\
&=\P(S_{t+1:\tau},O_{t+1:\tau}^{0:N},A_{t+1:\tau}|h_t^0,h,\ga_1,a,d^\prime_{t+1:T})\\
&\qquad\cdot\P(S_{\tau+1},O_{\tau+1}^{0:N},A_{\tau+1}|h_t^0,h,\ga_1,a,d^\prime_{t+1:T},S_{t+1:\tau},O_{t+1:\tau}^{0:N},A_{t+1:\tau})\\
&=\P(S_{t+1:\tau},O_{t+1:\tau}^{0:N},A_{t+1:\tau}|h_t^0,h,\ga_2,a,d^\prime_{t+1:T})\\
&\qquad\cdot\P(S_{\tau+1},O_{\tau+1}^{0:N},A_{\tau+1}|h_t^0,h,\ga_1,a,d^\prime_{t+1:T},S_{t+1:\tau},O_{t+1:\tau}^{0:N},A_{t+1:\tau})\tag{induction hypothesis}\\
&=\P(S_{t+1:\tau},O_{t+1:\tau}^{0:N},A_{t+1:\tau}|h_t^0,h,\ga_2,a,d^\prime_{t+1:T})\cdot\P(S_{\tau+1}|S_\tau,A_\tau)\cdot\P(O_{\tau+1}^{0:N}|S_{\tau+1})\\
&\qquad\cdot\I\left\{A_{\tau+1}=d^\prime_{\tau+1}\left(h_t^0,\ga_1,O_{t+1}^0,d^\prime_{t+1}(h_t^0,\ga_1,O_{t+1}^0),O_{t+2}^0,\dots,O_{\tau+1}^0\right)(h,a,O_{t+1}^{1:N},A_{t+1},\dots,O_{\tau+1}^{1:N})\right\}\\
&\overset{(\dagger)}{=}\P(S_{t+1:\tau},O_{t+1:\tau}^{0:N},A_{t+1:\tau}|h_t^0,h,\ga_2,a,d^\prime_{t+1:T})\cdot\P(S_{\tau+1}|S_\tau,A_\tau)\cdot\P(O_{\tau+1}^{0:N}|S_{\tau+1})\\
&\qquad\cdot\I\left\{A_{\tau+1}=d^\prime_{\tau+1}\left(h_t^0,\ga_2,O_{t+1}^0,d^\prime_{t+1}(h_t^0,\ga_2,O_{t+1}^0),O_{t+2}^0,\dots,O_{\tau+1}^0\right)(h,a,O_{t+1}^{1:N},A_{t+1},\dots,O_{\tau+1}^{1:N})\right\}\\
&=\P(S_{t+1:\tau},O_{t+1:\tau}^{0:N},A_{t+1:\tau}|h_t^0,h,\ga_2,a,d^\prime_{t+1:T})\\
&\qquad\cdot\P(S_{\tau+1},O_{\tau+1}^{0:N},A_{\tau+1}|h_t^0,h,\ga_2,a,d^\prime_{t+1:T},S_{t+1:\tau},O_{t+1:\tau}^{0:N},A_{t+1:\tau})\\
&=\P(S_{t+1:\tau+1},O_{t+1:\tau+1}^{0:N},A_{t+1:\tau+1}|h_t^0,h,\ga_2,a,d^\prime_{t+1:T}),
\end{align*}
where the equality in $(\dagger)$ holds due to the definition of policy $d^\prime$. Note that with the CI-based approach, a generic policy $d_t$ first maps an FCS $H_t^0$ to a prescription $\Ga_t$, which in term maps an FPS $H_t^{1:N}$ to an action $A_t$; therefore, $d_t(H_t^0)(H_t^{1:N})=\Ga_t(H_t^{1:N})=A_t$ refers to the final action $A_t$ under the policy $d_t$ and the supervisor's state $(H_t^0,H_t^{1:N})$. The claim implies that $\P(S_\tau,A_\tau|h_t^0,h,\ga_1,d^\prime_{t+1:T})=\P(S_\tau,A_\tau|h_t^0,h,\ga_2,d^\prime_{t+1:T})$ for all $\tau=t+1,\dots,T$, i.e. conditioning on $h_t^0,h,d^\prime_{t+1:T}$, the distribution of $(S_\tau,A_\tau)$ is exactly the same given $\ga_1$ or $\ga_2$; and $(S_\tau,A_\tau)$ where $\tau=t+1,\dots,T$ is what the expectations on both sides of $(*)$ are taken on.
\end{proof}

\begin{proof}[Proof of \pref{lem:two-ASPS}]
We preceed the proof by mathematical induction. The instantaneous part and the base case $t=T$ follow trivially from (ASPS2)
\begin{align*}
&\eqsp\left|\E\left[R_t(S_t,A_t)|h_t^0,h_1,\ga\right]-\E\left[R_t(S_t,A_t)|h_t^0,h_2,\ga\right]\right|\\
&\leq\left|\E\left[R_t(S_t,A_t)|h_t^0,h_1,\ga\right]-\E\left[R_t(S_t,A_t)|h_t^0,\zh,\ga\right]\right|+\left|\E\left[R_t(S_t,A_t)|h_t^0,\zh,\ga\right]-\E\left[R_t(S_t,A_t)|h_t^0,h_2,\ga\right]\right|\\
&\leq\ep_p/4+\ep_p/4\tag{(ASPS2)}\\
&=\ep_p/2.
\end{align*}
For the continuation part, we have
\allowdisplaybreaks
\begin{align*}
&\eqsp\E\left[V_{t+1}^S((H_t^0,\Ga_t,O_{t+1}^0),H_{t+1}^{1:N})|h_t^0,h_1,\ga\right]\\
&=\sum_{o_{t+1}^{0:N}}\P(o_{t+1}^{0:N}|h_t^0,h_1,\ga)V_{t+1}^S((h_t^0,\ga,o_{t+1}^0),(h_1,a,o_{t+1}^{1:N}))\\
&=\sum_{o_{t+1}^{0:N}}\sum_{s_{t+1}}\P(o_{t+1}^{0:N},s_{t+1}|h_t^0,h_1,\ga)V_{t+1}^S((h_t^0,\ga,o_{t+1}^0),(h_1,a,o_{t+1}^{1:N}))\\
&=\sum_{o_{t+1}^{0:N}}\sum_{s_{t+1}}\P(o_{t+1}^{0:N}|h_t^0,h_1,\ga,s_{t+1})\cdot\P(s_{t+1}|h_t^0,h_1,\ga)\cdot V_{t+1}^S((h_t^0,\ga,o_{t+1}^0),(h_1,a,o_{t+1}^{1:N}))\\
&=\sum_{o_{t+1}^{0:N}}\sum_{s_{t+1}}\P(o_{t+1}^{0:N}|s_{t+1})\cdot\P(s_{t+1}|h_t^0,h_1,\ga)\cdot V_{t+1}^S((h_t^0,\ga,o_{t+1}^0),(h_1,a,o_{t+1}^{1:N}))\tag{$\P_O$ specifies $O_{t+1}^{0:N}$ given $S_{t+1}$}\\
&=\sum_{o_{t+1}^{0:N}}\sum_{s_{t+1}}\P(o_{t+1}^{0:N}|s_{t+1})\cdot\P(s_{t+1}|h_t^0,h_1,a)\cdot V_{t+1}^S((h_t^0,\ga,o_{t+1}^0),(h_1,a,o_{t+1}^{1:N}))\tag{\pref{lem:gamma-to-a}}\\
&=\sum_{o_{t+1}^{0:N}}\P(o_{t+1}^{0:N}|h_t^0,h_1,a)V_{t+1}^S((h_t^0,\ga,o_{t+1}^0),(h_1,a,o_{t+1}^{1:N})),
\end{align*}
and the same equality holds for $h_2$. Therefore,
\allowdisplaybreaks
\begin{small}
\begin{align*}
&\eqsp\left|\E\left[V_{t+1}^S((H_t^0,\Ga_t,O_{t+1}^0),H_{t+1}^{1:N})|h_t^0,h_1,\ga\right]-\E\left[V_{t+1}^S((H_t^0,\Ga_t,O_{t+1}^0),H_{t+1}^{1:N})|h_t^0,h_2,\ga\right]\right|\\
&=\left|\sum_{o_{t+1}^{0:N}}\P(o_{t+1}^{0:N}|h_t^0,h_1,a)V_{t+1}^S((h_t^0,\ga,o_{t+1}^0),(h_1,a,o_{t+1}^{1:N}))-\sum_{o_{t+1}^{0:N}}\P(o_{t+1}^{0:N}|h_t^0,h_2,a)V_{t+1}^S((h_t^0,\ga,o_{t+1}^0),(h_2,a,o_{t+1}^{1:N}))\right|\\
&\leq\left|\sum_{o_{t+1}^{0:N}}\P(o_{t+1}^{0:N}|h_t^0,h_1,a)V_{t+1}^S((h_t^0,\ga,o_{t+1}^0),(h_1,a,o_{t+1}^{1:N}))-\sum_{o_{t+1}^{0:N}}\P(o_{t+1}^{0:N}|h_t^0,\zh,a)V_{t+1}^S((h_t^0,\ga,o_{t+1}^0),(h_1,a,o_{t+1}^{1:N}))\right|\\
&\es+\left|\sum_{o_{t+1}^{0:N}}\P(o_{t+1}^{0:N}|h_t^0,\zh,a)V_{t+1}^S((h_t^0,\ga,o_{t+1}^0),(h_1,a,o_{t+1}^{1:N}))-\sum_{o_{t+1}^{0:N}}\P(o_{t+1}^{0:N}|h_t^0,h_2,a)V_{t+1}^S((h_t^0,\ga,o_{t+1}^0),(h_1,a,o_{t+1}^{1:N}))\right|\\
&\es+\left|\sum_{o_{t+1}^{0:N}}\P(o_{t+1}^{0:N}|h_t^0,h_2,a)V_{t+1}^S((h_t^0,\ga,o_{t+1}^0),(h_1,a,o_{t+1}^{1:N}))-\sum_{o_{t+1}^{0:N}}\P(o_{t+1}^{0:N}|h_t^0,h_2,a)V_{t+1}^S((h_t^0,\ga,o_{t+1}^0),(h_2,a,o_{t+1}^{1:N}))\right|\\
&:=\textcircled{1}+\textcircled{2}+\textcircled{3}.
\end{align*}
\end{small}

For the first two terms, we have
\[
\textcircled{1},\textcircled{2}\leq2\|V_{t+1}\|_\infty\cdot\dl_p/8\leq T\Rb\dl_p/4
\]
by (ASPS3). Note that the above equation follows if $\Kc(\cdot,\cdot)$ is the total variation distance. If it is instead the Wasserstein metric, then the total variation distance will still be bounded by $\dl_p\big/\min_{x,y\in\Om(O_{t+1}^{0:N}),x\neq y}\|x-y\|$; we can redefine this value as $\dl_p$ so that the total variation distance is still bounded by $\dl_p$.

Now consider a fixed realization of $o_{t+1}^{0:N}$. We have
\begin{align*}
&\eqsp\vthh_{t+1}^{1:N}((h_t^0,\ga,o_{t+1}^0),(h_1,a,o_{t+1}^{1:N}))\\
&=\phih_{t+1}^{1:N}(\vthh_t^{1:N}(h_1),h_t^0,\ga,o_{t+1}^{0:N})\tag{(ASPS1)}\\
&=\phih_{t+1}^{1:N}(\vthh_t^{1:N}(h_2),h_t^0,\ga,o_{t+1}^{0:N})\tag{assumption}\\
&=\vthh_{t+1}^{1:N}((h_t^0,\ga,o_{t+1}^0),(h_2,a,o_{t+1}^{1:N}))\tag{(ASPS1)},
\end{align*}
so that under the public FCS $(h_t^0,\ga,o_{t+1}^0)$, the two FPSs $(h_1,a,o_{t+1}^{1:N})$ and $(h_2,a,o_{t+1}^{1:N})$ will be mapped to the same ASPS as well. Hence, by the induction hypothesis of \pref{lem:two-ASPS} which leads to \pref{cor:ASPS} at the $t+1$ step, we obtain
\begin{align*}
&\eqsp\left|V_{t+1}^S((h_t^0,\ga,o_{t+1}^0),(h_1,a,o_{t+1}^{1:N}))-V_{t+1}^S((h_t^0,\ga,o_{t+1}^0),(h_2,a,o_{t+1}^{1:N}))\right|\\
&\leq(T-t-1)(\ep_p+T\Rb\dl_p)/2+\ep_p/2.
\end{align*}
The last term can thus be bounded by
\begin{align*}
\textcircled{3}&\leq\sum_{o_{t+1}^{0:N}}\P(o_{t+1}^{0:N}|h_t^0,h_2,a)\left|V_{t+1}^S((h_t^0,\ga,o_{t+1}^0),(h_1,a,o_{t+1}^{1:N}))-V_{t+1}^S((h_t^0,\ga,o_{t+1}^0),(h_2,a,o_{t+1}^{1:N}))\right|\\
&\leq\sum_{o_{t+1}^{0:N}}\P(o_{t+1}^{0:N}|h_t^0,h_2,a)[(T-t-1)(\ep_p+T\Rb\dl_p)/2+\ep_p/2]\\
&=(T-t-1)(\ep_p+T\Rb\dl_p)/2+\ep_p/2.
\end{align*}
Combining the three terms plus the instantaneous part, it follows that
\begin{align*}
\left|Q_t^S(h_t^0,h_1,\ga)-Q_t^S(h_t^0,h_2,\ga)\right|&\leq\ep_p/2+2\cdot T\Rb\dl_p/4+(T-t-1)(\ep_p+T\Rb\dl_p)/2+\ep_p/2\\
&=(T-t)(\ep_p+T\Rb\dl_p)/2+\ep_p/2.
\end{align*}
\end{proof}

\begin{proof}[Proof of \pref{cor:ASPS}]
Assume the optimal prescription $\ga^*$ prescribes different actions on the two FPSs $h_1,h_2\in\Om(H_t^{1:N})$, so that $\ga^*(h_1)=a_1$ and $\ga^*(h_2)=a_2$ where $a_1\neq a_2$; otherwise, the claim directly follows by \pref{lem:two-ASPS}. Also, define $\ga^{\prime},\ga^{\prime\prime}\in\Om(\Ga_t)$ by
\[
\ga^{\prime}(h)=\left\{
\begin{array}{ll}
a_1&\text{if }h=h_2,\\
\ga^*(h)&\text{otherwise,}
\end{array}\right.
\qquad
\ga^{\prime\prime}(h)=\left\{
\begin{array}{ll}
a_2&\text{if }h=h_1,\\
\ga^*(h)&\text{otherwise.}
\end{array}\right.
\]
Let $\upsilon_1=Q_t^S(h_t^0,h_1,\ga^*)$ and $\upsilon_2=Q_t^S(h_t^0,h_2,\ga^*)$. Denote $B_t\triangleq(T-t)(\ep_p+T\Rb\dl_p)/2+\ep_p/2$ for simplicity. From \eqref{eq:supervisor-Q-expansion}, $Q_t(h_t,\ga^*)$ can be expanded as
\[
Q_t(h_t^0,\ga^*)=\sum_{h\neq h_1,h_2}\P(h|h_t^0)Q_t^S(h_t^0,h,\ga^*)+\P(h_1|h_t^0)\upsilon_1+\P(h_2|h_t^0)\upsilon_2.
\]
Likewise, we can also expand $Q_t(h_t,\ga^{\prime})$ to
\begin{align*}
&\eqsp Q_t(h_t^0,\ga^{\prime})\\
&=\sum_{h\neq h_1,h_2}\P(h|h_t^0)Q_t^S(h_t^0,h,\ga^{\prime})+\P(h_1|h_t^0)Q_t^S(h_t^0,h_1,\ga^{\prime})+\P(h_2|h_t^0)Q_t^S(h_t^0,h_2,\ga^{\prime})\\
&\geq\sum_{h\neq h_1,h_2}\P(h|h_t^0)Q_t^S(h_t^0,h,\ga^{\prime})+\P(h_1|h_t^0)Q_t^S(h_t^0,h_1,\ga^{\prime})+\P(h_2|h_t^0)\left[Q_t^S(h_t^0,h_1,\ga^{\prime})-B_t\right]\tag{\pref{lem:two-ASPS}}\\
&=\sum_{h\neq h_1,h_2}\P(h|h_t^0)Q_t^S(h_t^0,h,\ga^*)+\P(h_1|h_t^0)Q_t^S(h_t^0,h_1,\ga^*)+\P(h_2|h_t^0)\left[Q_t^S(h_t^0,h_1,\ga^*)-B_t\right]\tag{\pref{lem:two-prescription}}\\
&=\sum_{h\neq h_1,h_2}\P(h|h_t^0)Q_t^S(h_t^0,h,\ga^*)+\P(h_1|h_t^0)\upsilon_1+\P(h_2|h_t^0)(\upsilon_1-B_t);
\end{align*}
by symmetry
\[
Q_t(h_t^0,\ga^{\prime\prime})\geq\sum_{h\neq h_1,h_2}\P(h|h_t^0)Q_t^S(h_t^0,h,\ga^*)+\P(h_1|h_t^0)(\upsilon_2-B_t)+\P(h_2|h_t^0)\upsilon_2.
\]
We have
\allowdisplaybreaks
\begin{align*}
&\eqsp\sum_{h\neq h_1,h_2}\P(h|h_t^0)Q_t^S(h_t^0,h,\ga^*)+\P(h_1|h_t^0)\upsilon_1+\P(h_2|h_t^0)(\upsilon_1-B_t)\leq Q_t(h_t^0,\ga^{\prime})\\
&\leq Q_t(h_t^0,\ga^*)=\sum_{h\neq h_1,h_2}\P(h|h_t^0)Q_t^S(h_t^0,h,\ga^*)+\P(h_1|h_t^0)\upsilon_1+\P(h_2|h_t^0)\upsilon_2
\end{align*}
and
\begin{align*}
&\eqsp\sum_{h\neq h_1,h_2}\P(h|h_t^0)Q_t^S(h_t^0,h,\ga^*)+\P(h_1|h_t^0)(\upsilon_2-B_t)+\P(h_2|h_t^0)\upsilon_2\leq Q_t(h_t^0,\ga^{\prime\prime})\\
&\leq Q_t(h_t^0,\ga^*)=\sum_{h\neq h_1,h_2}\P(h|h_t^0)Q_t^S(h_t^0,h,\ga^*)+\P(h_1|h_t^0)\upsilon_1+\P(h_2|h_t^0)\upsilon_2.
\end{align*}
Canceling and rearranging the terms yield
\[
-B_t\leq\upsilon_1-\upsilon_2\leq B_t.
\]
\end{proof}

\begin{proof}[Proof of \pref{thm:ASPS}]
We prove the result by induction. The base case trivially follows from \pref{prop:ASPS-main}. Note that the continuation values at $T+1$ are defined to be $0$, i.e. $V_{T+1}(h_{T+1}^0)\triangleq0$ and $\Vh_{T+1}(h_{T+1}^0)\triangleq0$ for any $h_{T+1}^0\in\Om(H_{T+1}^0)$. Hence, for any $h_T^0\in\Om(H_T^0)$ and $\ga^*\in\Om(\Ga_t)$, we have
\[
Q_T(h_T^0,\ga^*)-\ep_p=V_T(h_T^0)-\ep_p\leq\max_{\lmh\in\Om(\Lmh_T)}Q_T(h_T^0,\ga_{\lmh,h_T^0})=\max_{\lmh\in\Om(\Lmh_T)}\Qh_T(h_T^0,\lmh)=\Vh_T(h_T^0).
\]
In the equation, $Q_T(h_T^0,\ga_{\lmh,h_T^0})=\Qh_T(h_T^0,\lmh)$ because there is no continuation value for $T$. Now for the induction step, we assume the induction hypothesis, i.e. the claim holds for some $t+1\leq T$ so that we have for any $h_{t+1}^0\in\Om(H_{t+1}^0)$,
\[
V_{t+1}(h_{t+1}^0)-\Vh_{t+1}(h_{t+1}^0)\leq\frac{(T-t-1)(T-t)}{2}(\ep_p+T\Rb\dl_p)+(T-t)\ep_p.
\]
\pref{prop:ASPS-main} states that for any $h_t^0\in\Om(H_t^0)$ and optimal prescription $\ga^*\in\argmax_\ga Q_t(h_t^0,\ga)$, there exists a $\lmh\in\Om(\Lmh_t)$ such that
\[
Q_t(h_t^0,\ga^*)-Q_t(h_t^0,\ga_{\lmh,h_t^0})\leq(T-t)(\ep_p+T\Rb\dl_p)+\ep_p.
\]
Write $\Cf_t\triangleq(T-t)(\ep_p+T\Rb\dl_p)+\ep_p$ for shorthand of notation. Then for this $\lmh$, we have
\bals
&\eqsp Q_t(h_t^0,\ga^*)-\Qh_t(h_t^0,\lmh)
=Q_t(h_t^0,\ga^*)-Q_t(h_t^0,\ga_{\lmh,h_t^0})+Q_t(h_t^0,\ga_{\lmh,h_t^0})-\Qh_t(h_t^0,\lmh)\\
&\leq\Cf_t+\E[R_t+V_{t+1}(H_{t+1}^0)|h_t^0,\ga_{\lmh,h_t^0}]-\E[R_t+\Vh_{t+1}(H_{t+1}^0)|h_t^0,\ga_{\lmh,h_t^0}]\\
&=\Cf_t+\sum_{h_{t+1}^0}\P(h_{t+1}^0|h_t^0,\ga_{\lmh,h_t^0})\left[V_{t+1}(h_{t+1}^0)-\Vh_{t+1}(h_{t+1}^0)\right]\\
&\leq\Cf_t+\sum_{h_{t+1}^0}\P(h_{t+1}^0|h_t^0,\ga_{\lmh,h_t^0})\left[\frac{(T-t-1)(T-t)}{2}(\ep_p+T\Rb\dl_p)+(T-t)\ep_p\right]\\
&=(T-t)(\ep_p+T\Rb\dl_p)+\ep_p+\frac{(T-t-1)(T-t)}{2}(\ep_p+T\Rb\dl_p)+(T-t)\ep_p\\
&=\frac{(T-t)(T-t+1)}{2}(\ep_p+T\Rb\dl_p)+(T-t+1)\ep_p.
\eals
\end{proof}
\section{Omitted Analysis in \pref{sec:ASCS}}\label{app:ASCS}

\begin{prop}\label{prop:ASCS-main}
Assume the reward function $R$ is uniformly bounded by $\Rb$. Let $h_1^0,h_2^0\in\Om(H_t^0)$ be two FCSs. If $\zh^0=\vthh_t^0(h_1^0)=\vthh_t^0(h_2^0)$, then for any $\lmh\in\Om(\Lmh_t)$,
\beq
\left|\Qh_t(h_1^0,\lmh)-\Qh_t(h_2^0,\lmh)\right|\leq2(T-t)(\ep_c+T\Rb\dl_c)+2\ep_c.
\eeq
\end{prop}

\begin{proof}
We proceed the proof again by mathematical induction. The instantaneous part as well as the base case $t=T$ trivially follow from (ASCS2)
\begin{align*}
&\eqsp\left|\E\left[R_t(S_t,A_t)|h_1^0,\lmh\right]-\E\left[R_t(S_t,A_t)|h_2^0,\lmh\right]\right|\\
&\leq\left|\E\left[R_t(S_t,A_t)|h_1^0,\lmh\right]-\E\left[R_t(S_t,A_t)|\zh^0,\lmh\right]\right|+\left|\E\left[R_t(S_t,A_t)|\zh^0,\lmh\right]-\E\left[R_t(S_t,A_t)|h_2^0,\lmh\right]\right|\\
&\leq\ep_c+\ep_c\tag{(ASCS2)}
=2\ep_c.
\end{align*}
For the continuation part in the induction step, we have
\allowdisplaybreaks
\begin{align*}
&\eqsp\left|\E\left[\Vh_{t+1}((H_t^0,\Lmh_t,O_{t+1}^0))|h_1^0,\lmh\right]-\E\left[\Vh_{t+1}((H_t^0,\Lmh_t,O_{t+1}^0))|h_2^0,\lmh\right]\right|\\
&=\left|\sum_{o_{t+1}^0}\P(o_{t+1}^0|h_1^0,\lmh)\Vh_{t+1}((h_1^0,\lmh,o_{t+1}^0))-\sum_{o_{t+1}^0}\P(o_{t+1}^0|h_2^0,\lmh)\Vh_{t+1}((h_2^0,\lmh,o_{t+1}^0))\right|\\
&\leq\left|\sum_{o_{t+1}^0}\P(o_{t+1}^0|h_1^0,\lmh)\Vh_{t+1}((h_1^0,\lmh,o_{t+1}^0))-\sum_{o_{t+1}^0}\P(o_{t+1}^0|\zh^0,\lmh)\Vh_{t+1}((h_1^0,\lmh,o_{t+1}^0))\right|\\
&\qquad+\left|\sum_{o_{t+1}^0}\P(o_{t+1}^0|\zh^0,\lmh)\Vh_{t+1}((h_1^0,\lmh,o_{t+1}^0))-\sum_{o_{t+1}^0}\P(o_{t+1}^0|h_2^0,\lmh)\Vh_{t+1}((h_1^0,\lmh,o_{t+1}^0))\right|\\
&\qquad+\left|\sum_{o_{t+1}^0}\P(o_{t+1}^0|h_2^0,\lmh)\Vh_{t+1}((h_1^0,\lmh,o_{t+1}^0))-\sum_{o_{t+1}^0}\P(o_{t+1}^0|h_2^0,\lmh)\Vh_{t+1}((h_2^0,\lmh,o_{t+1}^0))\right|\\
&:=\textcircled{1}+\textcircled{2}+\textcircled{3}.
\end{align*}
For the first two terms, we have
\[
\textcircled{1},\textcircled{2}\leq2\|\Vh_{t+1}\|_\infty\cdot\dl_c/2\leq T\Rb\dl_c
\]
by (ASCS3).

Now consider a fixed realization of $o_{t+1}^0$. We have
\begin{align*}
&\eqsp\vthh_{t+1}^0((h_1^0,\lmh,o_{t+1}^0))\\
&=\phih_{t+1}^0(\vthh_t^0(h_1),\lmh,o_{t+1}^0)\tag{(ASCS1)}\\
&=\phih_{t+1}^0(\vthh_t^0(h_2),\lmh,o_{t+1}^0)\tag{assumption}\\
&=\vthh_{t+1}^0((h_2^0,\lmh,o_{t+1}^0))\tag{(ASCS1)},
\end{align*}
so that the two FCSs (with ASPS-based prescription) $(h_1^0,\lmh,o_{t+1}^0)$ and $(h_2^0,\lmh,o_{t+1}^0)$ will be mapped to the same ASCS as well. Hence, by the induction hypothesis, we obtain
\[
\left|\Vh_{t+1}((h_1^0,\lmh,o_{t+1}^0))-\Vh_{t+1}((h_2^0,\lmh,o_{t+1}^0))\right|\leq2(T-t-1)(\ep_c+T\Rb\dl_c)+2\ep_c.
\]
The last term can thus be bounded by
\begin{align*}
\textcircled{3}
&\leq\sum_{o_{t+1}^0}\P(o_{t+1}^0|h_2^0,\lmh)\left|\Vh_{t+1}((h_1^0,\lmh,o_{t+1}^0))-\Vh_{t+1}((h_2^0,\lmh,o_{t+1}^0))\right|\\
&\leq\sum_{o_{t+1}^0}\P(o_{t+1}^0|h_2^0,\lmh)[2(T-t-1)(\ep_c+T\Rb\dl_c)+2\ep_c]
=2(T-t-1)(\ep_c+T\Rb\dl_c)+2\ep_c.
\end{align*}
Combining the three terms plus the instantaneous part, it follows that
\begin{align*}
\left|\Qh_t(h_1^0,\lmh)-\Qh_t(h_2^0,\lmh)\right|&\leq2\ep_c+2\cdot T\Rb\dl_c+2(T-t-1)(\ep_c+T\Rb\dl_c)+2\ep_c\\
&=2(T-t)(\ep_c+T\Rb\dl_c)+2\ep_c.
\end{align*}
\end{proof}

\begin{proof}[Proof of \pref{thm:ASCS}]
We proceed the proof again by mathematical induction. The base case $t=T$ trivially follows from (ASCS2). For the induction step, we have
\begin{align*}
&\eqsp\Qh_t(h_t^0,\lmh)-\widecheck{Q}_t(\vthh_t^0(h_t^0),\lmh)\\
&=\E\left[R_t(S_t,A_t)|h_t^0,\lmh\right]-\E\left[R_t(S_t,A_t)|\vthh_t^0(h_t^0),\lmh\right]\\
&\qquad+\E\left[\Vh_{t+1}(H_{t+1}^0)|h_t^0,\lmh\right]-\E\left[\widecheck{V}_{t+1}(\vthh_{t+1}^0(H_{t+1}^0))|\vthh_t^0(h_t^0),\lmh\right]\\
&\leq\ep_c+\E\left[\Vh_{t+1}(H_{t+1}^0)|h_t^0,\lmh\right]-\E\left[\widecheck{V}_{t+1}(\vthh_{t+1}^0(H_{t+1}^0))|\vthh_t^0(h_t^0),\lmh\right]\tag{(ASCS2)}\\
&=\ep_c+\E\left[\Vh_{t+1}(H_{t+1}^0)|h_t^0,\lmh\right]-\E\left[\Vh_{t+1}(H_{t+1}^0)|\vthh_t^0(h_t^0),\lmh\right]\\
&\qquad+\E\left[\Vh_{t+1}(H_{t+1}^0)|\vthh_t^0(h_t^0),\lmh\right]-\E\left[\widecheck{V}_{t+1}(\vthh_{t+1}^0(H_{t+1}^0))|\vthh_t^0(h_t^0),\lmh\right]\\
&=\ep_c+\sum_{h_{t+1}^0}\left[\P(h_{t+1}^0|h_t^0,\lmh)-\P(h_{t+1}^0|\vthh_t^0(h_t^0),\lmh)\right]\Vh_{t+1}(h_{t+1}^0)\\
&\qquad+\sum_{h_{t+1}^0}\P(h_{t+1}^0|\vthh_t^0(h_t^0),\lmh)\left[\Vh_{t+1}(h_{t+1}^0)-\widecheck{V}_{t+1}(\vthh_{t+1}^0(h_{t+1}^0))\right]\\
&:=\ep_c+\textcircled{1}+\textcircled{2}.
\end{align*}
The first term is bounded by (ASCS3)
\begin{align*}
\textcircled{1}&=\sum_{o_{t+1}^0}\left[\P(o_{t+1}^0|h_t^0,\lmh)-\P(o_{t+1}^0|\vthh_t^0(h_t^0),\lmh)\right]\Vh_{t+1}((h_t^0,\lmh,o_{t+1}^0))\\
&\leq2\|\Vh_{t+1}\|_\infty\cdot\dl_c/2\leq T\Rb\dl_c,
\end{align*}
while the second term can be bounded by the induction hypothesis
\allowdisplaybreaks
\begin{align*}
\textcircled{2}&\leq\sum_{h_{t+1}^0}\P(h_{t+1}^0|\vthh_t^0(h_t^0),\lmh)\left[(T-t-1)(\ep_c+T\Rb\dl_c)+\ep_c\right]\\
&=(T-t-1)(\ep_c+T\Rb\dl_c)+\ep_c.
\end{align*}
Combining the terms, it follows that
\begin{align*}
\Qh_t(h_t^0,\lmh)-\widecheck{Q}_t(\vthh_t^0(h_t^0),\lmh)&\leq\ep_c+T\Rb\dl_c+(T-t-1)(\ep_c+T\Rb\dl_c)+\ep_c\\
&=(T-t)(\ep_c+T\Rb\dl_c)+\ep_c.
\end{align*}
The $V$ part of the claim can be obtained by considering an optimal prescription $\lmh^*\in\argmax_{\lmh\in\Om(\Lmh_t)}\Qh_t(h_t^0,\lmh)$ in the $Q$ part.
\end{proof}

\section{Omitted Analysis in \pref{sec:comparison}}\label{app:comparison}

As mentioned in \pref{sec:comparison}, when considering $\ep_c=\dl_c=\ep_p=\dl_p=0$, we use SCS and SPS to refer to the common and private representations. Moreover, we use $Z$ and $\vth$ to denote the compressed state and the compression mapping when the error parameters are $0$, instead of $\Zh$ and $\vthh$.

\begin{proof}[Proof of \pref{prop:BCS=SCS}]
For (SCS1), BCSs can be updated recursively through Bayesian updates \citep{Demos_CI_2013}. For (SCS2), notice that
\[
\E[R_t(S_t,A_t)|h_t^0,\lm_t]=\sum_{s_t,h_t^{1:N}}\P(s_t,h_t^{1:N}|h_t^0)R(s_t,\ga_t(h_t^{1:N})),
\]
and the ensemble of $\P(s_t,h_t^{1:N}|h_t^0)$ through their spaces is exactly $\Pi_t(h_t^0)=\P(S_t,H_t^{1:N}|h_t^0)$. Similarly, it satisfies (SCS3) as well, since
\[
\P(O_{t+1}^0|h_t^0,\ga_t)=\sum_{s_t,h_t^{1:N}}\P(s_t,h_t^{1:N}|h_t^0)\cdot\sum_{s_{t+1}}\P(s_{t+1}|s_t,\ga_t(h_t^{1:N}))\cdot\P(O_{t+1}^0|s_{t+1}).
\]
The quantity $\Pi_t(h_t^0)=\P(S_t,H_t^{1:N}|h_t^0)$ again exactly encapsulates what is needed to compute $\P(O_{t+1}^0|h_t^0,\ga_t)$.
\end{proof}

\begin{proof}[Proof of \pref{prop:SPI-relation-1}]
\allowdisplaybreaks
\begin{align*}
&\eqsp\P(z_{t+1}^{1:N},o_{t+1}^0|h_t^0,h_t^{1:N},\ga_t,a_t)
=\sum_{o_{t+1}^{0:N}}\sum_{s_{t+1}}\P(z_{t+1}^{1:N},o_{t+1}^{0:N},s_{t+1}|h_t^0,h_t^{1:N},\ga_t,a_t)\\
&=\sum_{o_{t+1}^{0:N}}\sum_{s_{t+1}}\P(s_{t+1}|h_t^0,h_t^{1:N},\ga_t,a_t)\cdot\P(o_{t+1}^{0:N}|h_t^0,h_t^{1:N},\ga_t,a_t,s_{t+1})\cdot\P(z_{t+1}^{1:N}|h_t^0,h_t^{1:N},\ga_t,a_t,s_{t+1},o_{t+1}^{0:N})\\
&=\sum_{o_{t+1}^{0:N}}\sum_{s_{t+1}}\P(s_{t+1}|h_t^0,h_t^{1:N},\ga_t)\cdot\P(o_{t+1}^{0:N}|s_{t+1})\cdot\P(z_{t+1}^{1:N}|h_t^0,h_t^{1:N},\ga_t,s_{t+1},o_{t+1}^{0:N})\tag{redundancy of $a_t$ and $P_O$ specifies $O_{t+1}$ given $S_{t+1}$}\\
&=\sum_{o_{t+1}^{0:N}}\sum_{s_{t+1}}\P(s_{t+1}|h_t^0,h_t^{1:N},a_t)\cdot\P(o_{t+1}^{0:N}|s_{t+1})\cdot\P(z_{t+1}^{1:N}|h_t^0,h_t^{1:N},\ga_t,s_{t+1},o_{t+1}^{0:N})\tag{\pref{lem:gamma-to-a}}\\
&=\sum_{o_{t+1}^{0:N}}\sum_{s_{t+1}}\P(s_{t+1},o_{t+1}^{0:N}|h_t^0,h_t^{1:N},a_t)\cdot\I\{z_{t+1}^{1:N}=\phi_{t+1}^{1:N}(\vth_t(h_t^0,h_t^{1:N}),\ga_t,o_{t+1}^{0:N})\}\tag{(SPS1)}\\
&=\sum_{o_{t+1}^{0:N}}\P(o_{t+1}^{0:N}|h_t^0,h_t^{1:N},a_t)\cdot\I\{z_{t+1}^{1:N}=\phi_{t+1}^{1:N}(z_t^{1:N},\ga_t,o_{t+1}^{0:N})\}\\
&=\sum_{o_{t+1}^{0:N}}\P(o_{t+1}^{0:N}|h_t^0,z_t^{1:N},a_t)\cdot\I\{z_{t+1}^{1:N}=\phi_{t+1}^{1:N}(z_t^{1:N},\ga_t,o_{t+1}^{0:N})\}\tag{(SPS3)}\\
&=\P(z_{t+1}^{1:N},o_{t+1}^0|h_t^0,z_t^{1:N},\ga_t,a_t).
\end{align*}
Note the last equality follows as in (SPS3) it is implicitly assumed that $z_t^{1:N}=\vth_t(h_t^0,h_t^{1:N})$.
\end{proof}

\begin{proof}[Proof of \pref{prop:SPI-relation-2}]
\allowdisplaybreaks
\begin{align*}
&\eqsp\E[R(S_t,A_t)|h_t^0,h_t^n,a_t]=\sum_{s_t}R(s_t,a_t)\P(s_t|h_t^0,h_t^n)
=\sum_{s_t}R(s_t,a_t)\sum_{z_t^{-n}}\P(s_t,z_t^{-n}|h_t^0,h_t^n)\\
&=\sum_{s_t}R(s_t,a_t)\sum_{z_t^{-n}}\P(z_t^{-n}|h_t^0,h_t^n)\cdot\P(s_t|h_t^0,h_t^n,z_t^{-n})\\
&=\sum_{s_t}R(s_t,a_t)\sum_{z_t^{-n}}\P(z_t^{-n}|h_t^0,h_t^n)\sum_{h_t^{-n}}\P(s_t,h_t^{-n}|h_t^0,h_t^n,z_t^{-n})\\
&=\sum_{s_t}R(s_t,a_t)\sum_{z_t^{-n}}\P(z_t^{-n}|h_t^0,h_t^n)\sum_{h_t^{-n}}\P(h_t^{-n}|h_t^0,h_t^n,z_t^{-n})\cdot\P(s_t|h_t^0,h_t^n,z_t^{-n},h_t^{-n})\\
&=\sum_{s_t}R(s_t,a_t)\sum_{z_t^{-n}}\P(z_t^{-n}|h_t^0,h_t^n)\sum_{h_t^{-n}}\P(h_t^{-n}|h_t^0,h_t^n,z_t^{-n})\cdot\P(s_t|h_t^0,h_t^n,h_t^{-n})\tag{$z_t^{-n}=\vth_t^{-n}(h_t^0,h_t^{-n})$}\\
&=\sum_{z_t^{-n}}\P(z_t^{-n}|h_t^0,h_t^n)\sum_{h_t^{-n}}\P(h_t^{-n}|h_t^0,h_t^n,z_t^{-n})\sum_{s_t}R(s_t,a_t)\P(s_t|h_t^0,h_t^n,h_t^{-n})\\
&=\sum_{z_t^{-n}}\P(z_t^{-n}|h_t^0,h_t^n)\sum_{h_t^{-n}}\P(h_t^{-n}|h_t^0,h_t^n,z_t^{-n})\E[R(S_t,A_t)|h_t^0,h_t^{1:N},a_t]\\
&=\sum_{z_t^{-n}}\P(z_t^{-n}|h_t^0,h_t^n)\sum_{h_t^{-n}}\P(h_t^{-n}|h_t^0,h_t^n,z_t^{-n})
\E[R(S_t,A_t)|h_t^0,z_t^{1:N},a_t]\tag{(SPS2))}\\
&=\sum_{z_t^{-n}}\P(z_t^{-n}|h_t^0,h_t^n)\sum_{h_t^{-n}}\P(h_t^{-n}|h_t^0,h_t^n,z_t^{-n})\sum_{s_t}R(s_t,a_t)\P(s_t|h_t^0,z_t^n,z_t^{-n})\\
&=\sum_{s_t}R(s_t,a_t)\sum_{z_t^{-n}}\P(z_t^{-n}|h_t^0,h_t^n)\cdot\P(s_t|h_t^0,z_t^n,z_t^{-n})\sum_{h_t^{-n}}\P(h_t^{-n}|h_t^0,h_t^n,z_t^{-n})\\
&=\sum_{s_t}R(s_t,a_t)\sum_{z_t^{-n}}\P(z_t^{-n}|h_t^0,z_t^n)\cdot\P(s_t|h_t^0,z_t^n,z_t^{-n})\cdot1\tag{(SPI4)}\\
&=\sum_{s_t}R(s_t,a_t)\sum_{z_t^{-n}}\P(s_t,z_t^{-n}|h_t^0,z_t^n)=\sum_{s_t}R(s_t,a_t)\P(s_t|h_t^0,z_t^n)=\E[R(S_t,A_t)|h_t^0,z_t^n,a_t].
\end{align*}
Note that the superscript $-n$ only contains $[N]\sm\{n\}$ and does not contain $0$.
\end{proof}
\section{Algorithmic Framework}\label{app:framework}
In this section we propose an MARL algorithmic framework using the theory developed in \pref{sec:main-result}; the designing detail is left as future work. The framework adopts the ``centralized learning distributed execution" scheme, i.e., the agents assume the omniscient supervisor's view when they learn the compressions and policies.

\begin{algorithm}
\caption{Deep-MARL Framework}\label{alg:deep}
\nl Common part: coordinator computes (done in each agent in execution phase) $\Zh_t^0=\rho^0(O_t^0,\Lmh_{t-1})$, $\Lmh_t=\vphi^0(\Zh_t^0)$.\\
\nl Private part: agent $n$ computes $\Zh_t^n=\rho^n(O_t^n,A_{t-1}^n)$, $A_t^n=\vphi^n(\Zh_t^n,\Lmh_t^n)$.\\
\nl \If{in learning phase}{
\nl Coordinator computes $(\Rh_t,\Oh_{t+1}^0)=\psi^C(\Zh_t^0,\Lmh_t)$.\\
\nl $(\Rh_t,\Oh_{t+1}^0)$ is compared with ground truth $(R_t,O_{t+1}^0)$ and loss is back-propagated to $(\rho^0,\psi^C)$.\\
\nl Supervisor computes $(\Rh_t,\Oh_{t+1}^{0:N})=\psi^S(\Zh_t^{0:N},A_t^{1:N})$.\\
\nl $(\Rh_t,\Oh_{t+1}^{0:N})$ is compared with ground truth $(R_t,O_{t+1}^{0:N})$ and loss is back-propagated to $(\rho^{0:N},\psi^S)$.\\
\nl Coordinator computes $\sum_{\tau=t-W}^tR_\tau$ and performs policy gradient on $\vphi^{0:N}$.
}
\end{algorithm}

There are three types of functions within: the state networks $\rho^{0:N}$ modeled by recurrent neural networks (RNNs), the policy networks $\vphi^{0:N}$ modeled by deep neural networks (DNNs), and the prediction networks $\psi^C$ and $\psi^S$ also modeled by DNNs. The state networks $\rho^{0:N}$ serve the purpose of the compression mappings $\phih_t^{0:N}$ in \pref{dff:ASPS} and \pref{dff:ASCS}, and their recursive evolution structures suggest an RNN modeling. The common policy network $\vphi^0$ takes $\Zh_t^0$ as input and gives the prescription $\Lmh_t$ as suggested by \pref{sec:ASCS}; the private policy network $\vphi^n$ takes $\Zh_t^n$ and $\Lmh_t^n$ and outputs $A_t^n$. Here, we have to use a \emph{variable} to represent the \emph{prescription function}; hence, it cannot be directly applied to $\Zh_t^n$. Effective design of representing prescription function is left as future work, even though \pref{lem:two-prescription} provides a nice decomposition. Finally, the policy networks $\psi^C$ and $\psi^S$ are used to produce the predicted reward and new observations. In the learning phase, the predictions are compared with the ground truth and errors are back-propagated through the state and prediction networks. This requires full knowledge of $\Zh_t^{1:N}$ and $O_t^{1:N}$; consequently, the learning has to be centralized. A windowed (with length $W$) cumulative reward is summed for the computation of the loss in policy gradient methods \citep{RLBook_2018}, which is back-propagated through the policy networks; actor-critic methods can also be employed here. In the execution phase, only state and policy networks are required, and everything can be performed in a decentralized fashion. Note that in our proof of \pref{thm:ASPS} only the fact that $\Zh_t^{1:N}$ can be updated from $\Zh_{t-1}^{1:N}$ is needed.

To design a fully decentralized learning scheme, one needs conditions similar to (ASPS2) and (ASPS3) but only involving $o_t^n$, $h_t^n$ and $\zh_t^n$ instead of the whole $o_t^{1:N}$, $h_t^{1:N}$ and $\zh_t^{1:N}$ so that individual prediction networks that does not require supervisor's view can be designed. This might demand a ``consistency condition" similarly to (SPI4), and it is likely that this is only possible through agents communicating through some signal space or directly sending their parameters \citep{Zhang_NetMARLSurvey_2019}. Further, the required expectations over the realizations of the private information conditioned on the common information could be hard to estimate. This is also left as future work.

\end{document}